\documentclass{article}

\usepackage{arxiv}

\usepackage[utf8]{inputenc}
\usepackage[T1]{fontenc}

\usepackage{amsfonts}       
\usepackage[numbers, compress]{natbib}

\usepackage{listings}
\usepackage[most]{tcolorbox}
\tcbuselibrary{listings,breakable}

\usepackage{microtype}
\usepackage{graphicx}
\usepackage{subfigure}

\usepackage{hyperref}       
\usepackage{url}            
\usepackage{nicefrac} 

\usepackage{tabularx}
\usepackage{dsfont}
\usepackage{multirow}
\usepackage{wrapfig}

\usepackage{multirow}
\usepackage{xcolor}
\usepackage{lipsum,booktabs}
\usepackage[dvipsnames,table,xcdraw]{xcolor}

\usepackage{multicol}

\usepackage{caption}
\usepackage{subcaption}

\usepackage{enumitem}

\usepackage{upgreek}

\usepackage{array}
\usepackage{wrapfig}
\newcolumntype{C}[1]{>{\centering\arraybackslash}p{#1}}

\usepackage{amsmath,amsfonts,bm}

\def\eqref#1{equation~\ref{#1}}

\def\1{\bm{1}}

\DeclareMathAlphabet{\mathsfit}{\encodingdefault}{\sfdefault}{m}{sl}
\SetMathAlphabet{\mathsfit}{bold}{\encodingdefault}{\sfdefault}{bx}{n}

\usepackage{amsmath}
\usepackage{amsthm}
\usepackage{thmtools}
\usepackage{amssymb}
\usepackage{mathtools}
\usepackage{bbding}

\usepackage{makecell}  
\usepackage{adjustbox} 

\definecolor{brightmaroon}{rgb}{0.76, 0.13, 0.28}
\definecolor{darkcyan}{rgb}{0.0, 0.55, 0.55}

\definecolor{best}{HTML}{1F77B4}    
\definecolor{second}{HTML}{FF0000}  

\usepackage[capitalize,noabbrev]{cleveref}

\theoremstyle{plain}

\theoremstyle{definition}

\theoremstyle{remark}

\usepackage[textsize=tiny]{todonotes}

\title{Weight-Adjusted Gradients Reveal Parameter Importance and Failure Modes in LLMs}

\usepackage{authblk}

\author[1]{Shrestha Datta}
\author[2]{Hongfu Liu}
\author[1*]{Anshuman Chhabra}

\affil[1]{University of South Florida, Tampa, FL, USA}
\affil[2]{Brandeis University, Waltham, MA, USA}
\affil[*]{Corresponding Author}
\affil[ ]{\texttt{\{shresthadatta, anshumanc\}@usf.edu, hongfuliu@brandeis.edu}}

\renewcommand{\headeright}{}
\renewcommand{\undertitle}{}
\renewcommand{\shorttitle}{Weight-Adjusted Gradients Reveal Parameter Importance and Failure Modes in LLMs}

\date{}

\begin{document}

\maketitle

\begin{abstract}
    Understanding which parameters are influential in Large Language Models (LLMs) is central to improving their efficiency, reliability, and interpretability. We introduce Weight-Adjusted Gradients (WAG), a simple yet effective approach for estimating parameter importance that explicitly captures the interaction between model weights and first-order gradient information and identifies parameters that disproportionately influence model behavior, such as those responsible for collapse phenomena in LLMs. Across a range of models and settings, we show that WAG surfaces a tiny but critical subset of parameters ($\leq$ 0.5 parts per million or $0.00005\%$ of model size) whose modification leads to dramatic degradation in performance, indicating a novel failure mode. These findings also reveal a previously underexplored interplay between weights and gradients, suggesting that parameter importance cannot be fully understood through either signal alone. We demonstrate the practical utility of WAG across several diverse applications, such as expert allocation in Mixture-of-Experts (MoE) architectures, targeted unlearning, mixed-precision quantization, and layer selection for knowledge editing. In sum, WAG can serve as a unified approach for analyzing, debugging, and controlling LLMs, and opens new directions for principled parameter-level interpretation.
\end{abstract}

\section{Introduction}

Large Language Models (LLMs) have demonstrated strong performance across a wide range of tasks, becoming a central component in modern systems~\cite{brown2020language, touvron2023llama, openai2023gpt4}. While much of the progress in this area has focused on scaling model size, improving training procedures, and designing new architectures \cite{kaplan2020scaling, hoffmann2022training, minaee2024large}, comparatively less attention has been given to understanding the internal structure of trained models \cite{ferrando2024primer}. In particular, identifying which parameters are most influential to model behavior remains an important yet underexplored problem. Such understanding is critical for improving model reliability, enabling efficient adaptation, and supporting interpretability.

\looseness-1 A natural way to study this problem is through \textit{parameter importance} estimation. Intuitively, the importance of a parameter can be measured by evaluating how sensitive the model’s behavior is to perturbations of that parameter. Existing approaches typically rely on either parameter magnitudes or gradient-based signals. Magnitude-based methods, such as those used in pruning, assume that parameters with larger absolute values are more important~\cite{han2015learning, gale2019state}. Gradient-based approaches instead measure sensitivity through first-order information, for example using saliency or influence approximations~\cite{lecun1990optimal, molchanov2017pruning, vitel2025first}. Second-order methods further incorporate curvature information, but often incur high computational cost and rely on approximations that may not scale well to modern LLMs~\cite{hassibi1993second, frantar2022gptq, askari2026layerif}. Despite these efforts, the underlying relationship between weight and gradient information remains poorly understood, and existing methods offer little insight into how this interaction relates to catastrophic failure modes such as severe performance degradation. 

\looseness-1 In this paper, we examine parameter importance from the perspective of the interaction between parameter values and gradient signals. We observe that treating these two sources of information independently can overlook parameters that are both large in magnitude and highly sensitive to optimization dynamics. Motivated by this, we propose \textit{Weight-Adjusted Gradients (\textsc{WAG})}, a simple metric that combines zeroth-order and first-order information to estimate parameter importance. The formulation is straightforward and can be computed efficiently during standard training or evaluation without requiring second-order information or retraining.

\looseness-1 Our analysis shows that \textsc{WAG} identifies a \textit{tiny} subset of parameters that have a disproportionate impact on model performance. In particular, perturbing these parameters leads to significant degradation and model collapse. We find that many of these parameters are not captured by existing magnitude-based or gradient-based metrics, suggesting that important aspects of model behavior are missed when considering these signals in isolation. Thus, in this work we study \textsc{WAG} from a theoretical and empirical lens, and subsequently, demonstrate the effectiveness and usefulness of \textsc{WAG} in several practical settings.


\looseness-1\textbf{Contributions}. We summarize our main contributions in this paper as follows:\vspace{-1.25mm}
\begin{itemize}[leftmargin=*, nosep]
\item \looseness-1 We propose \textit{Weight-Adjusted Gradients (\textsc{WAG})} that leverages the unexplored relation between weights and gradients to estimate parameter influence and identify parameters, when masked, lead to significant performance degradation. Without requiring expensive computation, WAG captures the sensitivity of loss to multiplicative parameter perturbations.
\vspace{1mm}
\item We undertake empirical analyses across several LLMs and datasets, and also provide theoretical interpretations for the parameters captured by \textsc{WAG} and the reason why masking these induces a catastrophic degradation in performance.
\vspace{1mm}
\item Finally, we demonstrate the effectiveness and applicability of WAG with competitive performance in diverse practical applications, including expert allocation in Mixture-of-Experts (MoE) architectures, targeted unlearning, mixed-precision quantization, and knowledge editing.
\end{itemize}

\section{Related Works}

\looseness-1 We now review the literature along three aspects. First, we discuss how parameter importance has been quantified in prior work. Next, we examine how these measures have been operationalized in downstream applications. Finally, we discuss past work on failure modes in LLMs, specifically model collapse. We also provide a more comprehensive literature review in Appendix \ref{app:related_works} due to space constraints.

\textbf{General Parameter Importance.} 
Existing approaches primarily characterize parameter importance from three perspectives: (1) magnitude-based importance, (2) gradient-based sensitivity, and (3) curvature-based measures. Magnitude-based approaches include utilizing the weights and activations \cite{hanson1988comparing, gale2019state, han2015learning, sun2024simple}, since they have a direct impact on the output. Gradients estimate parameter importance by measuring sensitivity to parameter removal \cite{molchanov2017pruning}. Curvature-based approaches further account for second-order information, such as Optimal Brain Damage through Hessian diagonal approximation \cite{lecun1990optimal}, Optimal Brain Surgeon (OBS) through full Hessian \cite{hassibi1993second}, and Optimal Brain Apoptosis through computationally efficient Hessian \cite{sun2025optimal}.

\textbf{Applications of Estimating Parameter Influence. }
The parameter importance estimates are widely operationalized in pruning methods, which remove parameters or structured groups based on their estimated contribution to model performance. Recent approaches for LLMs include magnitude itself \cite{han2015learning}, magnitude-activation based Wanda \cite{sun2024simple}, gradient-based criteria \cite{das2023beyond, yang2025wanda++, zimmer2025sparseswaps}. The parameter-importance measures have also been applied to layer-wise sparsity allocation. For example, AlphaPruning \cite{lu2024alphapruning} uses spectral characteristics of layer weights, OWL \cite{yin2023outlier} relies on activation outliers, and LayerIF \cite{askari2026layerif} estimates layer-wise influence.

Beyond pruning and sparsity allocation, influence-based layer quality estimates have also been applied to continual learning \cite{chhabra2025oga} for catastrophic forgetting prevention. Regularization-based methods, Elastic Weight Consolidation \cite{kirkpatrick2017overcoming}, Synaptic Intelligence \cite{zenke2017continual}, and Memory Aware Synapses \cite{aljundi2018memory} address this by penalizing changes to parameters deemed important. Recent work \cite{gao2025defying} provides a unified theoretical account of this family by formally defining parameter importance via influence functions.

\looseness-1 Uniform allocation of LoRA experts across transformer layers has been questioned by recent work showing that expert redundancy and routing overfitting degrade performance~\cite{gao2025mola}. AlphaLoRA~\cite{qing2024alphalora} addresses this using spectral layer quality metrics to inform allocation. LayerIF~\cite{askari2026layerif} instead derives task-specific allocation budgets from layer-wise influence.

Parameter influence is also used in the task of mixed-precision quantization, where parameter groups with the most saliency are preserved in high precision and the others in low precision. Recent studies, TACQ \cite{xiao2025taskcircuit}, PB-LLM \cite{shang2024pbllm} and SliM-LLM \cite{huang2025slimllm}, employ parameter magnitude and OBS criteria as a measure of the saliency of the groups. Moreover, gradient-based methods like Salient Layers Editing Model \cite{mishra-etal-2024-correcting} and Gradient Layer Analysis \cite{datta2026golden} are used for important editing layer identification in knowledge editing.

\textbf{Identifying Failure Modes in LLMs. }
Failure due to model collapse in LLMs broadly manifests in two forms: collapse arising from corrupted training data, and collapse arising from architectural modification. Recent work has characterized the phenomenon of model collapse arising from recursive training on synthetic data without access to human-generated data \cite{shumailov2023curse}. Recent work \cite{wang2024language} shows that recursive training introduces error accumulation across generations, leading to irreversible model collapse. At the architectural level, a recent study \cite{shi2026understanding} found that performance remains stable until pruning reaches the \textit{Silent Phase}, after which collapse is abrupt and irreversible. On the other hand, recent studies show that collapse is driven by a sparse set of critical parameters \cite{reda2025many}. However, these works do not identify the critical parameters or the minimum number of weights needed to induce model collapse.

\section{The WAG Metric: Theory and Empirical Insights}
This section introduces \textsc{WAG}, its theoretical foundation as a parameter sensitivity measure under multiplicative perturbations, and extensive empirical analysis.

\subsection{Weight-Adjusted Gradients}

\looseness-1 We introduce \textit{Weight-Adjusted Gradients (WAG)}, a parameter importance metric that combines parameter magnitude with first-order gradient information. Magnitude-based methods assume that parameters with larger absolute values contribute more strongly to model performance. While effective in some settings, they ignore local loss sensitivity and may assign high importance to functionally inactive parameters. Conversely, gradient-based methods capture instantaneous sensitivity but may overemphasize parameters with large gradients despite having negligible scale or limited functional impact. WAG combines these complementary signals by measuring their interaction. Formally, for a parameter vector $\theta \in \mathbb{R}^d$ and loss function $\ell(z;\theta)$ evaluated on a test sample $z$, we define the WAG score of parameter $\theta_i$ as:
\begin{equation}
\boxed{
\mathrm{WAG}_i
=
-
\theta_i
\frac{\partial \ell(z;\theta)}{\partial \theta_i}.
}
\end{equation}

The negative sign aligns the score with the local descent direction of the loss. The metric can be computed per parameter, aggregated across layers, or accumulated across datasets depending on the application (e.g. averaging WAG values).

WAG is closely related to several classical ideas in optimization and pruning. The quantity $\theta_i \frac{\partial \ell}{\partial \theta_i}$ appears implicitly in first-order Taylor approximations for network saliency and pruning~\cite{lecun1990optimal, molchanov2017pruning, lee2019snip}. Related notions have also been studied in second-order saliency methods and sensitivity-based pruning approaches~\cite{hassibi1993second, gale2019state}. However, prior work has treated these interaction terms as auxiliary approximations for pruning objectives rather than as a direct object of study for understanding parameter importance in large-scale models such as LLMs.


\subsection{Theoretical Analysis}
\label{sec:theoretical_analysis}

\looseness-1 Theoretically, we interpret \textsc{WAG} through a \textit{log-space parameterization} and show that it characterizes local sensitivity to multiplicative parameter perturbations, identifying directions of greatest sensitivity. This provides a view of trained LLMs as \textit{scale-balanced} systems, where the most fragile coordinates have the largest absolute \textsc{WAG} scores. Detailed proofs for all theorems are provided in Appendix~\ref{app:proofs}.

\textbf{\textit{WAG as a Gradient in Log-Parameter Space}.}
We first show that WAG arises naturally when expressing optimization in logarithmic parameter coordinates.

\begin{restatable}[WAG as Log-Parameter Gradient]{theorem}{logGradient}
\label{thm:log_gradient}
Assume $\theta_i \neq 0$ and define the logarithmic parameterization:
\begin{equation*}
u_i = \log |\theta_i|.
\end{equation*}
Then the gradient of the loss with respect to $u_i$ satisfies:
\begin{equation*}
\frac{\partial \ell(z;\theta)}{\partial u_i}
=
\theta_i
\frac{\partial \ell(z;\theta)}{\partial \theta_i}.
\end{equation*}
Equivalently:

\begin{equation*}
\mathrm{WAG}_i
=
-
\frac{\partial \ell(z;\theta)}{\partial u_i}.
\end{equation*}
\end{restatable}



Theorem~\ref{thm:log_gradient} shows that WAG is not solely a heuristic combination of weights and gradients. In fact, it exactly corresponds to the gradient of the loss in logarithmic parameter coordinates. Thus, WAG measures sensitivity to relative parameter changes rather than \textit{additive} perturbations. We prove this relative sensitivity claim next.

\textbf{\textit{First-Order Sensitivity Under Multiplicative Perturbations}.}
We can show that WAG naturally characterizes first-order loss variation under multiplicative parameter perturbations.

\begin{restatable}[First-Order Multiplicative Sensitivity]{theorem}{multiplicativeSensitivity}
\label{thm:multiplicative}
Let $\theta' \in \mathbb{R}^d$ be obtained through multiplicative perturbations of the type:
\begin{equation*}
\theta_i' = \theta_i(1+\epsilon_i),
\end{equation*}
where $\epsilon_i$ are sufficiently small. Then the corresponding first-order change in loss satisfies:
\begin{equation*}
\ell(z;\theta') - \ell(z;\theta)
=
-
\sum_{i=1}^d
\epsilon_i \mathrm{WAG}_i
+
\mathcal{O}(\|\epsilon\|^2).
\end{equation*}
\end{restatable}



Theorem~\ref{thm:multiplicative} shows that WAG exactly characterizes the first-order effect of relative parameter perturbations on the loss. Parameters with large-magnitude WAG scores therefore correspond to directions that induce disproportionately large functional changes under multiplicative modification.

\textbf{\textit{WAG and Scale-Balanced LLMs}.}
The previous results establish that WAG can be interpreted in two complementary ways. Theorem~\ref{thm:log_gradient} shows that WAG is exactly the gradient of the loss in logarithmic parameter coordinates, while Theorem~\ref{thm:multiplicative} shows that WAG characterizes first-order sensitivity under multiplicative parameter perturbations. These results suggest that trained neural networks are naturally organized in terms of parameter \textit{scales} rather than raw parameter values. Trained LLMs have parameter magnitudes that vary widely across layers and blocks \cite{an2025systematic}, yet behave stably because the loss depends on how these scales are coordinated \cite{dinh2017sharp, ba2016layer}. Unlike additive probes of this geometry \cite{foret2021sharpness}, WAG measures sensitivity to \textit{relative} scale changes. Based on this insight, we next show that WAG identifies the direction along which this coordination is most fragile.

\begin{restatable}[WAG Characterizes the Most Sensitive Log-Scale Direction]{theorem}{scaleBalance}
\label{thm:scale_balance}

Let $w = (\mathrm{WAG}_1,\ldots,\mathrm{WAG}_d) \in \mathbb{R}^d$
denote the WAG vector, and consider multiplicative parameter perturbations:
$\theta_i' = \theta_i(1+\epsilon_i),$
with perturbation vector:
$\epsilon = (\epsilon_1,\ldots,\epsilon_d).$
%
%
Then, among all multiplicative perturbations satisfying:
$\|\epsilon\|_2=r,$
the maximal first-order increase in loss is $r\|w\|_2,$ and is attained uniquely by perturbations parallel to $-w.$
Equivalently, the WAG vector defines the direction of greatest local sensitivity in logarithmic parameter space.
\end{restatable}







\looseness-1 Therefore, no multiplicative perturbation with the same norm can produce a larger first-order increase in loss. Since Theorem~\ref{thm:log_gradient} shows that WAG is the gradient in logarithmic parameter coordinates, it follows that the WAG vector defines the direction of greatest local sensitivity in log-parameter space. The norm $\|w\|_2$ measures the corresponding worst-case sensitivity, while coordinates with large $|\mathrm{WAG}_i|$ contribute most strongly to this direction.


This provides a geometric interpretation of \textit{scale balance}. Trained models can be viewed as configurations whose parameter scales jointly maintain low loss, with WAG identifying the directions in which this balance is most fragile. In sum, parameters with large $|\mathrm{WAG}_i|$ are the most sensitive to local changes in scale, which helps explain why perturbing or masking them can cause disproportionate performance degradation, as we show empirically in Section~\ref{sec:emp_analysis}. We provide a more detailed discussion on this alongside the proof in Appendix~\ref{app:scalebalance_proof}.

\begin{figure*}[t]
    \centering
    \includegraphics[width=0.85\textwidth]{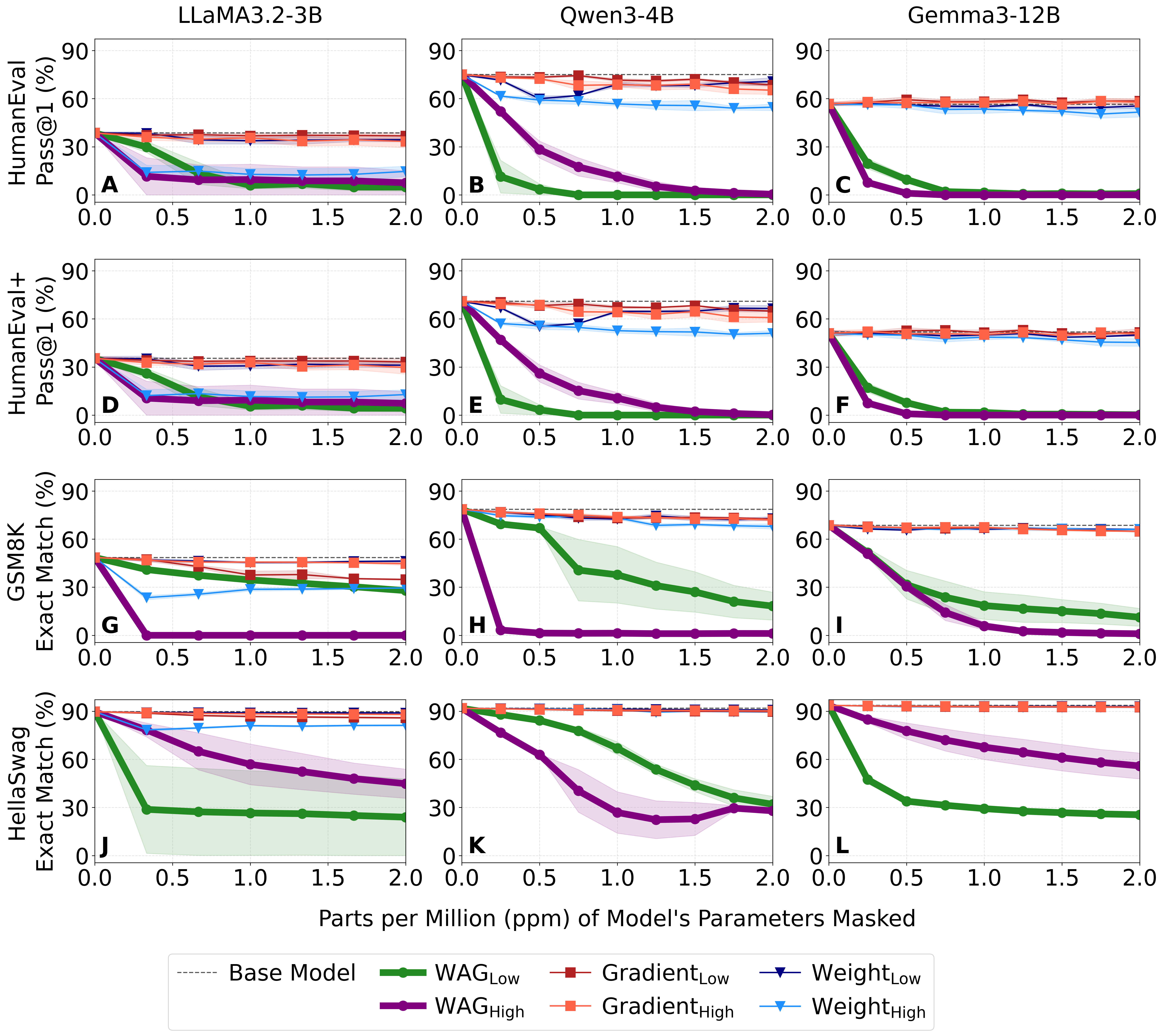}\vspace{-2mm}
    \caption{Visualization of model collapse due to masking just $\leq 2$\ parts per million (i.e., $\leq 0.0002\%$) parameters across LLaMA3.2-3B (A, D, G and J), Qwen3-4B (B, E, H and K), and Gemma3-12B (C, F, I and L) models evaluated on \textit{HumanEval} (A, B and C), \textit{HumanEval+} (D, E and F), \textit{GSM8K} (G, H and I), and \textit{HellaSwag} (J, K, and L) datasets. The performance is measured using \textit{Pass@1} and \textit{exact match} accuracy. Parameters are selected based on having lowest gradient $\rightarrow$ Gradient$_\text{Low}$, highest gradient $\rightarrow$ Gradient$_\text{High}$, lowest weight $\rightarrow$ Weight$_\text{Low}$, highest weight $\rightarrow$ Weight$_\text{High}$, lowest \textsc{WAG} $\rightarrow$ WAG$_\text{Low}$, and highest \textsc{WAG} $\rightarrow$ WAG$_\text{High}$ values respectively. Masking the tiny number of WAG$_\text{Low}$ and WAG$_\text{High}$ identified parameters leads to rapid model collapse.}
    \label{fig:high_low}\vspace{-4mm}
\end{figure*}

\subsection{Empirical Analysis}\label{sec:emp_analysis}

\textbf{Impact of masking \textsc{WAG} parameters with extreme scores.} \label{sec:empirical}
\looseness-1 To empirically assess the impact of \textsc{WAG}-identified parameters, we examine the effect of masking parameters with the \textit{lowest}, \textit{highest}, and \textit{largest-magnitude} \textsc{WAG} values, denoted as WAG$_\text{Low}$, WAG$_\text{High}$, and $|\textsc{WAG}|$, respectively. We compare these effects with gradient- and weight-based parameter selection to examine whether the observed collapse behavior of \textsc{WAG} is inherited from \textsc{WAG}'s constituent components. We evaluate free-form generation performance on coding, math, and commonsense reasoning tasks using \textit{HumanEval} \cite{chen2021evaluating}, \textit{HumanEval+} \cite{evalplus}, \textit{GSM8K} \cite{cobbe2021training}, and \textit{HellaSwag} \cite{zellers2019hellaswag} datasets on LLaMA3.2-3B \cite{touvron2023llama}, Qwen3-4B \cite{yang2025qwen3}, and Gemma3-12B \cite{team2025gemma}.

\looseness-1 Figure \ref{fig:high_low} indicates that the WAG$_\text{Low}$ and WAG$_\text{High}$ identified parameters behave distinctly compared to magnitude and gradient identified parameters. Masking a tiny number of parameters based on \textsc{WAG} leads to performance degradation at a higher rate and this effect is consistent across several LLMs and datasets. For LLaMA3.2-3B on most datasets, we can see that average performance reaches near-zero exact match at just $\approx$1,500 masked parameters (or 0.5 ppm). A similar trend is also observed for $|\text{WAG}|$ and, due to space limitations, is presented in Appendix \ref{app:quantitaive_collapse}.


\textbf{Qualitative analysis of model collapse induced by WAG-parameter removal.}
We also qualitatively inspect the generations of the corresponding base model and the masked models to study this collapse. We present an example for Qwen3-4B on \textit{GSM8K} in Figure \ref{fig:qual-abswaggsm8k}, and owing to space constraints, provide detailed model continuations for Qwen3-4B on \textit{HumanEval} (4,000 or about 1e-4\% parameters masked) and \textit{GSM8K} (8,000 or about 2e-4\% parameters masked) in Appendix~\ref{app:qualitative}. As can be seen, \textsc{WAG}-based masking causes both generative failure and loss of reasoning correctness. In coding, the model fails to generate a function body, while in mathematical reasoning, it enters repetition loops (e.g., repeating \texttt{Greg} interminably as in Figure \ref{fig:qual-abswaggsm8k}) or produces fluent but corrupted reasoning. For instance, as shown in Appendix~\ref{app:qualitative} for WAG$_\text{Low}$, the masked model divides \$20 among three rather than four siblings, omitting Greg himself. Moreover, this behavior is specific to \textsc{WAG} as under the same masking budget, gradient- and weight-based variants retain fluent, correct solutions.

\begin{figure}[t]
\centering\footnotesize
\begin{tcolorbox}[colback=gray!4,colframe=gray!40,boxrule=0.3pt,left=2pt,right=2pt,top=2pt,bottom=2pt]
\textbf{Prompt: }\emph{Walking with his 3 younger siblings, Greg found \$20 and split it equally. How much did each get?}
\end{tcolorbox}
\begin{minipage}{0.9\textwidth}
\begin{tcblisting}{
  % title={Base Model (Exact match 78.5\%)},
  title={Base Model},
  colback=green!3, colframe=green!50!black,
  fonttitle=\bfseries\footnotesize,
  listing only,
  listing options={escapeinside={(*@}{@*)}},
  top=-5pt, bottom=-5pt,
  % listing options={basicstyle=\ttfamily\scriptsize,breaklines=true,columns=fullflexible}
}
There were 1 + 3 = 4 of them who found the money.
Each got 20 / 4 = $5.   ####  5
(*@\textcolor{blue}{\textbf{Observation:} \textit{Correct answer.}}@*)
\end{tcblisting}

\begin{tcblisting}{
  % title={abs-WAG High (Exact match 1.1\%)},
  title={|WAG|},
  colback=red!3, colframe=red!40,
  fonttitle=\bfseries\footnotesize,
  listing only,
  listing options={escapeinside={(*@}{@*)}},
  top=-5pt, bottom=-5pt,
  % listing options={basicstyle=\ttfamily\scriptsize,breaklines=true,columns=fullflexible}
}
GregGregGregGregGregGregGregGregGregGregGregGregGregGregGreg
GregGregGregGregGregGregGregGregGregGregGregGregGregGreg...
(*@\textcolor{blue}{\textbf{Observation:} \textit{Token "Greg" repeats to 2048 chars; the extracted answer is the same loop.}}@*)
\end{tcblisting}

\begin{tcblisting}{
  % title={Gradient / Weight representative (Exact match 22--71\%)},
  title={Gradient / Weight representative },
  colback=blue!3, colframe=blue!40,
  fonttitle=\bfseries\footnotesize,
  listing only,
  listing options={escapeinside={(*@}{@*)}},
  top=-5pt, bottom=-5pt,
  % listing options={basicstyle=\ttfamily\scriptsize,breaklines=true,columns=fullflexible}
}
There are 4 of them because 1 + 3 = 4.
Each of them gets $5 because 20 / 4 = 5.   ####  5  
(*@\textcolor{blue}{\textbf{Observation:} \textit{All six gradient/weight variants recover the answer.}}@*)
\end{tcblisting}\vspace{-3mm}
\end{minipage}
\caption{Qualitative example of model failure caused by masking 2 ppm = 0.0002\% $|\text{WAG}|$-identified model parameters. Verbatim Qwen3-4B generations on \textit{GSM8K} at 8,000 masked parameters indicate token repetition (``Greg\ldots''). Gradient- and weight-masked variants generate the correct solution.}
\label{fig:qual-abswaggsm8k}
\vspace{-5mm}
\end{figure}

\begin{wrapfigure}[15]{r}{0.62\linewidth}
\vspace{-5.5mm}
\centering
\includegraphics[width=\linewidth]{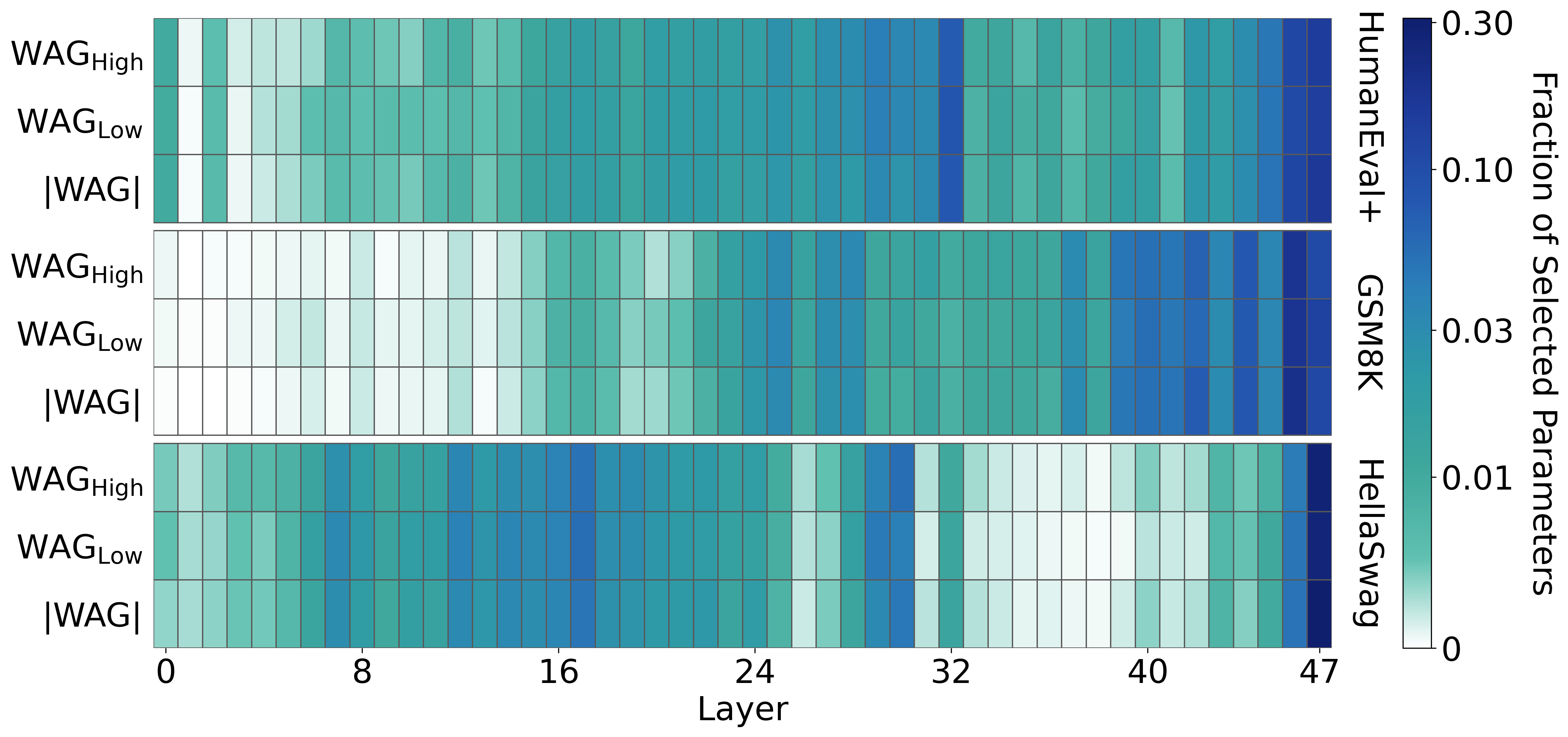}\vspace{-3.7mm}
\caption{{\looseness-1 Distribution of 2 parts per million or 0.0002\% selected parameters by WAG$_\text{Low}$, WAG$_\text{High}$, and $|\textsc{WAG}|$ across layers of Gemma3-12B for \textit{HumanEval+}, \textit{GSM8K} and \textit{HellaSwag}.}}
\label{fig:wagparameter_distribution}
\vspace{-0pt}
\end{wrapfigure}

\textbf{\textsc{WAG} selected parameter distribution across layers.}
\looseness-1 For more insight into where the masked 0.0002\% parameters identified by \textsc{WAG} metrics fall across the model layers, we present the distribution across layers as a heatmap in Figure ~\ref{fig:wagparameter_distribution}. We see that the majority of the selected tiny set of parameters are located in the last layers, irrespective of the task. However, the fraction is still very low compared to the total size of the layer. At most, 6,665 parameters are identified by $|\textsc{WAG}|$ from a single last layer of size 224M parameters. Thus, only a maximum of 0.0003\% of parameters are being masked from a single layer. We see a similar trend across the other LLMs, and present results for these in Appendix \ref{app:wag_distribution}. Moreover, we provide additional visualizations of the WAG scores across layers in Appendix \ref{app:wag_score_distribution}. It can be observed that the average \textsc{WAG} scores across layers exhibit a task-specific distribution. Finally, note that early-to-middle layers typically exhibit higher scores.

\section{Utilizing WAG in Diverse Applications}\vspace{-0mm}

We now demonstrate how WAG's ability to characterize local sensitivity of the loss under multiplicative parameter perturbations can be leveraged across several downstream applications. Specifically, we examine: (A) expert allocation in Mixture-of-Experts (MoE) models, (B) efficient and effective machine unlearning through targeted layer updates, (C) mixed-precision quantization via the identification of precision-sensitive weight groups, and (D) layer selection for effective knowledge editing. We evaluate the applications based on free-form generation for LLaMA3.2-3B and Gemma3-12B models. We provide all details of the performance metrics and the experimental implementation for each application in Appendices \ref{app:perf_metrics} and \ref{app:exp_setup}, respectively.

\subsection{Application A: \textsc{WAG} for Mixture-of-Experts Allocation}
\label{sec:expert_allocation}

\looseness-1 We apply \textsc{WAG} to allocate a fixed budget of $N$ parallel experts across the layers of a Mixture-of-Experts model. We allocate more experts to layers containing parameters that are more sensitive to perturbations, as identified by \textsc{WAG}, to better capture the layer's complexity. We instantiate this principle with our three prior criteria: WAG$_\text{Low}$, WAG$_\text{High}$, and $|\textsc{WAG}|$.

We evaluate the \textsc{WAG}-based allocations against state-of-the-art approaches such as Heavy-Tailed Self-Regularization theory-driven AlphaLoRA \cite{qing2024alphalora} and fixed heuristic MoLA configurations \cite{gao2025mola}: triangle (2 4 6 8), rectangle (5 5 5 5), and hourglass (8 2 2 8). We report zero-shot performance on \textit{MRPC} \cite{dolan2005automatically}, \textit{CoLA} \cite{wang2018glue}, and \textit{ScienceQA} \cite{lu2022learn} using LLaMA3.2-3B and Gemma3-12B.

\begin{table*}[t]
\centering
\caption{Comparison of \textsc{WAG}-guided expert allocation performance compared against AlphaLoRA and MoLA configurations (triangle, rectangle, hourglass) assessed using \textit{Zero-shot} accuracy (\%) over 3 seeds, evaluated on \textit{MRPC}, \textit{CoLA}, and \textit{ScienceQA}. The total expert budget is fixed ($N=140$ for LLaMA3.2-3B, $N=240$ for Gemma3-12B) with top-$2$ routing. Best average performance is in \colorbox{green!20}{\textbf{green}} and second best is \colorbox{green!20}{green}.}
\label{tab:moe}
\vspace{-7pt}
\resizebox{0.95\textwidth}{!}{%
\begin{tabular}{l| c ccc ccc}
\toprule
\multirow{2}{*}{\textbf{Dataset}} & \multirow{2}{*}{\textbf{AlphaLoRA}} & \multicolumn{3}{c}{\textbf{MoLA}} & \multicolumn{3}{c}{\textbf{WAG}} \\
\cmidrule(lr){3-5} \cmidrule(lr){6-8}
 & & Triangle & Rectangle & Hourglass & WAG$_\text{Low}$ & WAG$_\text{High}$ & $|\text{WAG}|$ \\
\midrule
\multicolumn{8}{l}{LLaMA3.2-3B} \\
\midrule
MRPC & 47.17$_{\pm7.279}$ & 52.41$_{\pm1.337}$ & 50.78$_{\pm2.320}$ & 55.15$_{\pm1.117}$ & 54.94$_{\pm6.354}$ & 54.47$_{\pm1.239}$ & 54.59$_{\pm3.229}$ \\
CoLA & 78.62$_{\pm0.838}$ & 78.33$_{\pm1.578}$ & 77.47$_{\pm1.237}$ & 76.93$_{\pm1.175}$ & 79.58$_{\pm0.671}$ & 79.23$_{\pm0.975}$ & 79.99$_{\pm0.305}$ \\
ScienceQA & 53.04$_{\pm0.390}$ & 50.75$_{\pm0.378}$ & 51.60$_{\pm1.817}$ & 50.24$_{\pm1.669}$ & 51.09$_{\pm0.859}$ & 51.65$_{\pm1.701}$ & 46.22$_{\pm1.531}$ \\
\midrule
Average & 59.61$_{\pm2.265}$ & 60.50$_{\pm0.198}$ & 59.95$_{\pm1.327}$ & 60.77$_{\pm0.364}$ & \colorbox{green!20}{\textbf{61.87$_{\pm2.049}$}} & \colorbox{green!20}{61.78$_{\pm0.628}$} & 60.27$_{\pm1.054}$ \\
\midrule
\midrule
\multicolumn{8}{l}{Gemma3-12B} \\
\midrule
MRPC & 69.66$_{\pm1.418}$ & 70.80$_{\pm3.828}$ & 73.41$_{\pm2.351}$ & 73.43$_{\pm1.919}$ & 71.05$_{\pm2.432}$ & 73.78$_{\pm0.504}$ & 76.44$_{\pm1.045}$ \\
CoLA & 82.42$_{\pm0.565}$ & 83.83$_{\pm0.692}$ & 83.41$_{\pm0.418}$ & 81.72$_{\pm0.338}$ & 83.51$_{\pm0.111}$ & 83.80$_{\pm0.055}$ & 83.41$_{\pm0.534}$ \\
ScienceQA & 92.49$_{\pm0.170}$ & 92.57$_{\pm0.182}$ & 93.18$_{\pm0.208}$ & 93.06$_{\pm0.159}$ & 93.36$_{\pm0.197}$ & 92.54$_{\pm0.045}$ & 93.33$_{\pm0.173}$ \\
\midrule
Average & 81.53$_{\pm0.291}$ & 82.40$_{\pm1.280}$ & 83.33$_{\pm0.844}$ & 82.74$_{\pm0.612}$ & 82.64$_{\pm0.849}$ & \colorbox{green!20}{83.37$_{\pm0.164}$} & \colorbox{green!20}{\textbf{84.40$_{\pm0.367}$}} \\
\bottomrule
\end{tabular}%
}
\vspace{-10pt}
\end{table*}

\textbf{Result.} \noindent
Table~\ref{tab:moe} shows that \textsc{WAG}-guided allocation consistently outperforms both AlphaLoRA and all fixed MoLA configurations on average. On LLaMA3.2-3B, allocation based on the WAG$_\text{Low}$ scores achieves the best average accuracy of 61.87\%, surpassing the strongest baseline MoLA hourglass by 1.10\%, with the largest gain on \textit{MRPC}. On Gemma3-12B, the $|\text{WAG}|$ criterion performs best with an average of 84.40\%, outperforming the strongest baseline MoLA rectangle by 1.07\%. Notably, WAG$_\text{High}$ remains a strong and stable criterion across both models, ranking second on both LLaMA3.2-3B and Gemma3-12B. Overall, sensitivity-aware expert allocation using \textsc{WAG} is more effective than both MoLA and AlphaLoRA.

\subsection{Application B: \textsc{WAG} for Targeted Unlearning}
\label{sec:application_unlearning}

\looseness-1 We next apply \textsc{WAG} to \textit{targeted unlearning} to remove the influence of a designated \textit{forget set} from a trained LLM. We build on the Gradient Difference (GradDiff) \cite{liu2022continual} unlearning algorithm, which updates all model parameters. We hypothesize that the information to be removed is concentrated in layers containing the most sensitive parameters to the forget set, which can be identified by WAG$_\text{Low}$, WAG$_\text{High}$, and $|\textsc{WAG}|$. Thus, restricting updates to layers identified by \textsc{WAG} should (i) improve unlearning quality by concentrating updates to important layers, and (ii) reduce computational cost since we are only updating a fraction of the model.

\begin{wrapfigure}[16]{r}{0.42\linewidth}
\vspace{-9.7mm}
\centering
\includegraphics[width=\linewidth]{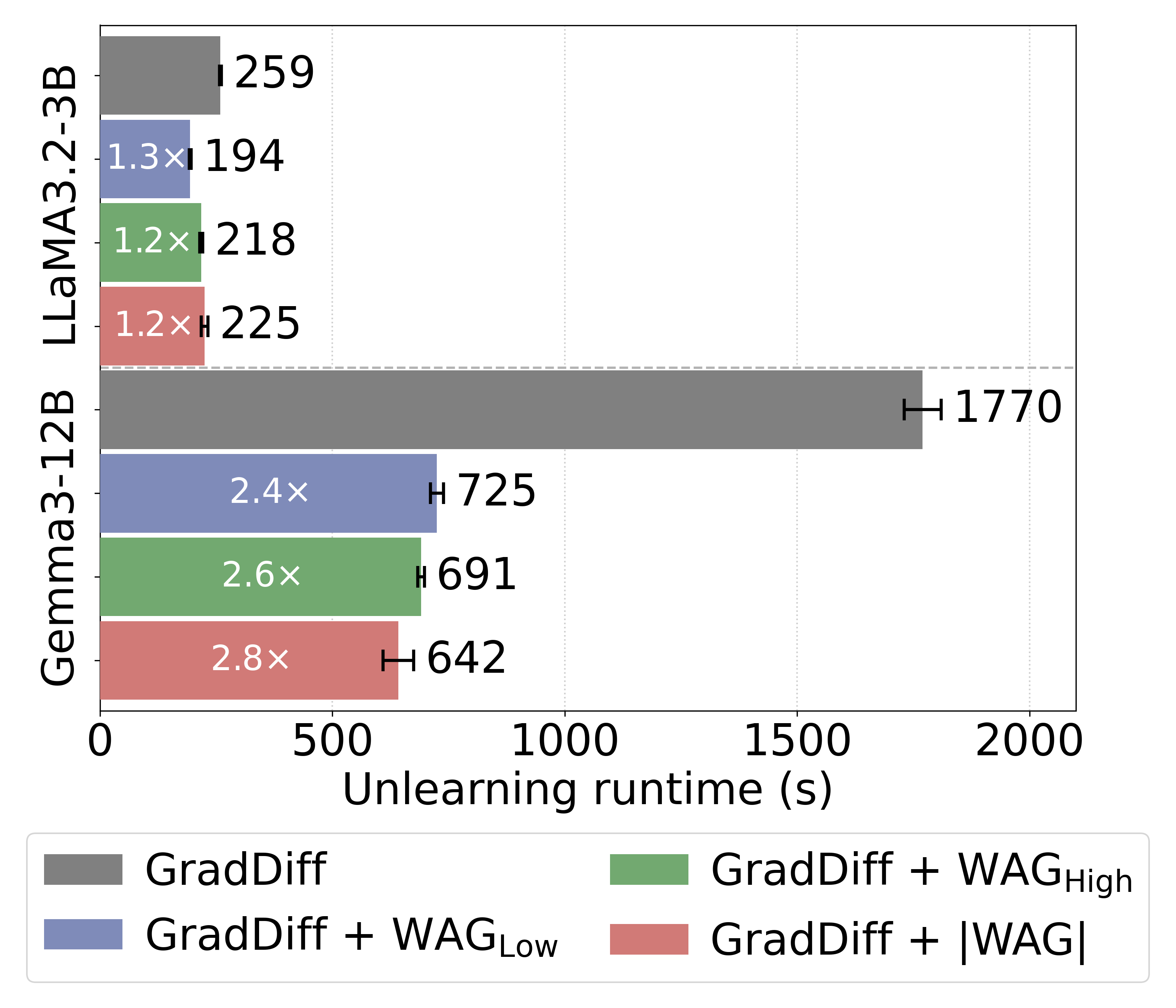}\vspace{-1mm}
\vspace{-17pt}
\caption{{\looseness-1 Unlearning runtime in seconds (3 runs) for methods. \textsc{WAG}-based unlearning is $\sim$1.9$\times$ faster than GradDiff on average.}}
\label{fig:runtime_analysis}
\end{wrapfigure}

We evaluate on the \textit{TOFU} benchmark \cite{maini2024tofu}, using the \textit{10\% forget} split, and compare against GradDiff across models. For each \textsc{WAG} criterion, we select the top 75\% of layers based on the number of most \textsc{WAG}-sensitive parameters they contain and apply GradDiff only to those layers.

\begin{table}[t]
  \centering
  \caption{We compare unlearning performance on the \textit{TOFU} forget set using standard GradDiff and \textsc{WAG}-targeted GradDiff, which updates only the top 75\% of layers selected by WAG$_\text{Low}$, WAG$_\text{High}$, and $|\text{WAG}|$ respectively, across LLaMA3.2-3B and Gemma3-12B models. We report the performance metrics \textit{Model Utility}, \textit{Forget Truth Ratio (TR)}, and \textit{KS Forget} using the mean $\pm$ standard error of the mean over 3 seeds. Best performance is in \colorbox{green!20}{\textbf{green}} and second best in \colorbox{green!20}{{green}}.}\vspace{0.5mm}
  \label{tab:tofu_wag_75pct_3run}
  \begin{minipage}{0.83\linewidth}
  \centering
  \renewcommand{\arraystretch}{1.1}
  \setlength{\tabcolsep}{5pt}
  \resizebox{1\linewidth}{!}{%
  \begin{tabular}{@{}l|l|cccc@{}}
    \toprule
    \multirow{2}{*}{\textbf{Model}} & \multirow{2}{*}{\textbf{Metric}}
      & \multicolumn{4}{c}{\textbf{Method}} \\
    \cmidrule(l){3-6}
    & & \textbf{GradDiff}
      & \textbf{GradDiff + WAG$_\text{Low}$}
      & \textbf{GradDiff + WAG$_\text{High}$}
      & \textbf{GradDiff + $|$WAG$|$} \\
    \midrule
    \multirow{3}{*}{LLaMA3.2-3B}
      & Model Utility & 0.565$_{\pm0.014}$ & \colorbox{green!20}{\textbf{0.578$_{\pm0.008}$}} & \colorbox{green!20}{0.575$_{\pm0.006}$} & 0.562$_{\pm0.002}$ \\
      & Forget TR     & 0.497$_{\pm0.056}$ & \colorbox{green!20}{0.546$_{\pm0.043}$} & 0.541$_{\pm0.023}$ & \colorbox{green!20}{\textbf{0.589$_{\pm0.006}$}} \\
      & KS Forget     & 0.588$_{\pm0.014}$ & \colorbox{green!20}{0.607$_{\pm0.015}$} & 0.599$_{\pm0.013}$ & \colorbox{green!20}{\textbf{0.613$_{\pm0.007}$}} \\
    \midrule
    \multirow{3}{*}{Gemma3-12B}
      & Model Utility & 0.532$_{\pm0.039}$ & 0.549$_{\pm0.010}$ & \colorbox{green!20}{0.567$_{\pm0.040}$} & \colorbox{green!20}{\textbf{0.589$_{\pm0.042}$}} \\
      & Forget TR     & 0.289$_{\pm0.031}$ & 0.311$_{\pm0.014}$ & \colorbox{green!20}{0.340$_{\pm0.004}$} & \colorbox{green!20}{\textbf{0.360$_{\pm0.008}$}} \\
      & KS Forget     & 0.449$_{\pm0.012}$ & 0.468$_{\pm0.139}$ & \colorbox{green!20}{\textbf{0.513$_{\pm0.004}$}} & \colorbox{green!20}{0.490$_{\pm0.007}$} \\
    \bottomrule
  \end{tabular}%
  }
  \end{minipage}\vspace{-3mm}
\end{table}

\textbf{Result.} 
Table~\ref{tab:tofu_wag_75pct_3run} shows that \textsc{WAG}-targeted GradDiff consistently improves upon standard GradDiff. The improvement is more stable for $|\textsc{WAG}|$.  Averaged across both LLMs and all three \textsc{WAG} criteria, targeted unlearning improves \textit{Forget Truth Ratio}, \textit{KS Forget}, and \textit{Model Utility} by 14.5\%, 6.2\%, and 4.0\%, respectively, relative to GradDiff. The largest gains in \textit{Forget Truth Ratio} are achieved by $|\textsc{WAG}|$, with improvements of 18.51\% on LLaMA3.2-3B and 24.57\% on Gemma3-12B. Notably, on Gemma3-12B, $|\textsc{WAG}|$ improves all three metrics simultaneously. 

Beyond unlearning performance, \textsc{WAG}-targeted GradDiff also reduces unlearning runtime cost. As shown in Figure~\ref{fig:runtime_analysis}, it is 1.2$\times$ and 2.6$\times$ faster than standard GradDiff on LLaMA3.2-3B and Gemma3-12B, respectively. Our results demonstrate that \textsc{WAG} enables more \textit{effective} and \textit{efficient} unlearning.

\subsection{Application C: \textsc{WAG} for Mixed Precision Quantization}
\label{sec:application_quantization}

\looseness-1 We apply \textsc{WAG} to mixed-precision quantization to identify important weight groups to retain in high precision while quantizing the rest. Following the SliM-LLM protocol \cite{huang2025slimllm}, we score each group by the average $|\textsc{WAG}|$ of its parameters, with higher scores indicating greater importance for preservation. We compare this criterion with random selection, parameter magnitude, and Optimal Brain Surgeon (OBS) \cite{hassibi1993second} saliency from SparseGPT \cite{frantar2023sparsegpt}.
We quantize models trained on coding tasks using the SliM-LLM approach, preserving the top-$\kappa$ ($\kappa \in {5\%, 10\%}$) parameter groups in each weight matrix to FP16 and quantize the remaining to 4-bit. We evaluate \textit{HumanEval} and \textit{HumanEval+} pass@1 on LLaMA3.2-3B and Gemma3-12B.

\begin{table*}[t]
\centering
\caption{Mixed-precision quantization results using {SliM-LLM} with the top $n\%$ ($5\%/10\%$) parameters preserved at high precision FP16, and rest quantized to low precision 4-bit. Evaluation is conducted on \textit{HumanEval} / \textit{HumanEval+} pass@1 (\%) over 3 seeds. Compared to baselines, WAG attains improved performance (top \textit{Pass@1} accuracy (\%) in \colorbox{green!20}{\textbf{green}}; second best in \colorbox{green!20}{green}).}
\label{tab:fp16w4_select}
\vspace{-7pt}
\resizebox{1\textwidth}{!}{%
\begin{tabular}{l c| ccccc}
\toprule
\textbf{Models \& Datasets} & \textbf{Unquantized \%} & \textbf{Random} & \textbf{Magnitude} & \textbf{OBS} & \textbf{$|\text{WAG}|$} & \textbf{OBS+$|\text{WAG}|$} \\
\midrule
\multicolumn{7}{l}{LLaMA3.2-3B} \\
\midrule
\multirow{2}{*}{HumanEval} & 5\% & 30.49$_{\pm2.464}$ & 30.08$_{\pm1.939}$ & 35.98$_{\pm0.931}$ & 35.98$_{\pm2.199}$ & 38.01$_{\pm0.407}$ \\
 & 10\% & 30.69$_{\pm2.344}$ & 31.71$_{\pm1.863}$ & 36.18$_{\pm0.203}$ & 37.80$_{\pm1.408}$ & 37.20$_{\pm0.704}$ \\
\multirow{2}{*}{HumanEval+} & 5\% & 27.24$_{\pm2.666}$ & 27.24$_{\pm1.664}$ & 33.33$_{\pm1.236}$ & 32.93$_{\pm1.760}$ & 35.37$_{\pm0.352}$ \\
 & 10\% & 28.46$_{\pm2.473}$ & 29.67$_{\pm1.466}$ & 33.54$_{\pm0.610}$ & 34.15$_{\pm1.269}$ & 33.94$_{\pm0.538}$ \\
\midrule
\multicolumn{2}{c}{Average} & 29.22$_{\pm2.364}$ & 29.67$_{\pm1.661}$ & 34.76$_{\pm0.535}$ & \colorbox{green!20}{35.21$_{\pm1.630}$} & \colorbox{green!20}{\textbf{36.13$_{\pm0.403}$}} \\
\midrule
\midrule
\multicolumn{7}{l}{Gemma3-12B} \\
\midrule
\multirow{2}{*}{HumanEval} & 5\% & 54.88$_{\pm0.610}$ & 55.89$_{\pm0.538}$ & 56.10$_{\pm1.760}$ & 55.28$_{\pm2.473}$ & 55.89$_{\pm0.203}$ \\
 & 10\% & 56.10$_{\pm0.610}$ & 55.49$_{\pm1.220}$ & 54.88$_{\pm1.760}$ & 56.50$_{\pm1.423}$ & 57.32$_{\pm0.610}$ \\
\multirow{2}{*}{HumanEval+} & 5\% & 48.37$_{\pm0.733}$ & 50.41$_{\pm0.813}$ & 49.59$_{\pm1.016}$ & 51.02$_{\pm2.473}$ & 49.80$_{\pm0.538}$ \\
 & 10\% & 49.80$_{\pm0.203}$ & 48.98$_{\pm1.132}$ & 49.19$_{\pm1.587}$ & 50.20$_{\pm2.063}$ & 51.42$_{\pm0.203}$ \\
\midrule
\multicolumn{2}{c}{Average} & 52.29$_{\pm0.152}$ & 52.69$_{\pm0.366}$ & 52.44$_{\pm1.504}$ & \colorbox{green!20}{53.25$_{\pm1.252}$} & \colorbox{green!20}{\textbf{53.61$_{\pm0.221}$}} \\

\bottomrule
\end{tabular}%
}
\vspace{-5pt}
\end{table*}

\textbf{Result.} \noindent
From Table~\ref{tab:fp16w4_select} we observe that the \textsc{WAG} metric-guided criteria improve over random group selection across both models and budgets. On average, $|\textsc{WAG}|$ alone surpasses random, magnitude, and OBS saliency on both LLMs. Importantly, augmenting the strong OBS baseline with \textsc{WAG} (OBS$+|\textsc{WAG}|$) yields a further overall performance improvement of 0.64\% in comparison to $|\textsc{WAG}|$.

\subsection{Application D: \textsc{WAG} for Knowledge Editing}
\label{sec:application_ke}

\looseness-1 We finally apply \textsc{WAG} to knowledge editing under the locate-then-edit \cite{wang2024easyediteasytouseknowledgeediting} paradigm, where the goal is to inject a new fact while preserving unrelated knowledge. A central challenge is identifying where to edit: targeting the wrong layers can either fail to incorporate the new fact or induce unrelated behavior \cite{basani2026exposing}. We therefore target layers containing parameters most sensitive to the specific fact being edited, as identified by \textsc{WAG}.
We evaluate R-ROME~\cite{gupta2024rebuilding} on \textit{ZSRE} \cite{levy2017zero}, \textit{WikiCounterFact} \cite{zhang2024comprehensive}, \textit{WikiBio} \cite{hartvigsen2023aging}, and \textit{WikiRecent} \cite{cohen2024evaluating}, following the proxy-set protocol of \citet{datta2026golden}. We use $|\textsc{WAG}|$ to rank parameters by sensitivity to the edit and select the MLP layer containing the highest concentration of such parameters. We compare our approach with the standard Causal Mediation Analysis (CMA) \cite{geva2023dissecting} layer selection for knowledge editing. Selection is performed using a small proxy set, while results are evaluated on an unseen test set.

\begin{table}[t]
\centering
\caption{\looseness-1 R-ROME Knowledge editing performance comparison of $|\text{WAG}|$ and CMA on four datasets: \textit{ZSRE}, \textit{WikiCounterFact}, \textit{WikiBio}, and \textit{WikiRecent}. Different datasets support different performance metrics, where \textit{Rewrite Accuracy} $\rightarrow$ RwA, \textit{Rephrase Accuracy} $\rightarrow$ RpA, \textit{Locality} $\rightarrow$ LOC, \textit{Portability} $\rightarrow$ PRT, \textit{Fluency} $\rightarrow$ FLC, and \textit{Overall} $\rightarrow$ OV. Best overall performance is denoted as \colorbox{green!20}{\textbf{green}}.}
\footnotesize
\begin{minipage}{0.50\textwidth}
\centering
\textit{ZSRE}
\medskip
{
\setlength{\tabcolsep}{1.04pt}
\resizebox{\linewidth}{!}{
\begin{tabular}{c | c | c c c c c}
\toprule
Model & Selection & RwA & RpA & LOC & PRT & OV\\
\midrule
\multirow{2}{*}{\shortstack[c]{LLaMA3.2-3B}}
& CMA & 0.9946 & 0.9595 & 0.9420 & 0.5383 & 0.8586 \\
& |WAG| & 0.9906 & 0.9483 & 0.9624 & 0.5433 & \colorbox{green!20}{\textbf{0.8612}} \\
\midrule
\multirow{2}{*}{Gemma3-12B}
& CMA & 0.8581 & 0.7041 & 0.9421 & 0.4989 & 0.7508 \\
& |WAG| & 0.9575 & 0.8150 & 0.9646 & 0.5028 & \colorbox{green!20}{\textbf{0.8100}} \\
\bottomrule
\end{tabular}
}
}
\end{minipage}%
\hfill
\begin{minipage}{0.50\textwidth}
\centering
\textit{WikiCounterFact}
\medskip
{
\setlength{\tabcolsep}{1.04pt}
\resizebox{\linewidth}{!}{
\begin{tabular}{c | c | c c c c c}
\toprule
Model & Selection & RwA & RpA & LOC & PRT & OV\\
\midrule
\multirow{2}{*}{\shortstack[c]{LLaMA3.2-3B}}
& CMA & 0.9954 & 0.3570 & 0.5583 & 0.8811 & 0.7019 \\
& |WAG| & 0.9550 & 0.3807 & 0.5736 & 0.8936 & \colorbox{green!20}{\textbf{0.7038}} \\
\midrule
\multirow{2}{*}{Gemma3-12B}
& CMA & 0.8276 & 0.5265 & 0.4473 & 0.4871 & 0.5757 \\
& |WAG| & 0.9290 & 0.5746 & 0.4771 & 0.4871 & \colorbox{green!20}{\textbf{0.6209}} \\
\bottomrule
\end{tabular}
}
}
\end{minipage}%
\vspace{-1.5mm}
\begin{minipage}{0.50\textwidth}
\centering
\textit{WikiBio}
\medskip
{
\setlength{\tabcolsep}{3.2pt}
\resizebox{\linewidth}{!}{
\begin{tabular}{c | c | c c c c}
\toprule
Model & Selection & RwA & LOC & FLC & OV\\
\midrule
\multirow{2}{*}{\shortstack[c]{LLaMA3.2-3B}}
& CMA & 0.8904 & 0.6625 & 0.8889 & \colorbox{green!20}{\textbf{0.8139}} \\
& |WAG| & 0.7898 & 0.7351 & 0.8853 & 0.8034 \\
\midrule
\multirow{2}{*}{Gemma3-12B}
& CMA & 0.6850 & 0.6566 & 0.6703 & 0.6706 \\
& |WAG| & 0.8153 & 0.8062 & 0.6670 & \colorbox{green!20}{\textbf{0.7628}} \\
\bottomrule
\end{tabular}
}
}
\end{minipage}%
\hfill
\begin{minipage}{0.50\textwidth}
\centering
\textit{WikiRecent}
\medskip
{
\setlength{\tabcolsep}{1.04pt}
\resizebox{\linewidth}{!}{
\begin{tabular}{c | c | c c c c c}
\toprule
Model & Selection & RwA & LOC & PRT & FLC & OV\\
\midrule
\multirow{2}{*}{\shortstack[c]{LLaMA3.2-3B}}
& CMA & 0.9844 & 0.4976 & 0.5362 & 0.8994 & \colorbox{green!20}{\textbf{0.7294}} \\
& |WAG| & 0.9844 & 0.4976 & 0.5362 & 0.8994 & \colorbox{green!20}{\textbf{0.7294}} \\
\midrule
\multirow{2}{*}{Gemma3-12B}
& CMA & 0.9710 & 0.5796 & 0.5216 & 0.5020 & 0.6436 \\
& |WAG| & 0.9775 & 0.5779 & 0.5327 & 0.5013 & \colorbox{green!20}{\textbf{0.6474}} \\
\bottomrule
\end{tabular}
}
}
\end{minipage}%
\vspace{2mm}
\vspace{2mm}
\label{tab:ke}\vspace{-8mm}
\end{table}

\textbf{Result.} \noindent
Table~\ref{tab:ke} shows that $|\text{WAG}|$ improves the \textit{Overall} score by 3.38\% relative to CMA, averaged across all models and datasets. The gains are most pronounced for \textit{Locality} and \textit{Portability}, where $|\text{WAG}|$ matches or exceeds CMA in every setting. In sum, $|\text{WAG}|$ serves as an effective layer-selection mechanism for locate-then-edit knowledge editing.

\section{Conclusion}
\looseness-1 In this work, we studied the problem of identifying parameters in LLMs that most influence model behavior. While existing approaches estimate parameter importance/sensitivity using their magnitudes, first-order gradients, and/or expensive second-order approximations, they have not explored the underlying relationship between weights and gradients. In contrast, we found that this unexplored relationship, which we term Weight-Adjusted Gradients (\textsc{WAG}) identifies a tiny subset of parameters that have a disproportionate impact on model performance, leading to model collapse. We provided theoretical interpretations for \textsc{WAG} and comprehensive empirical analyses to study this novel failure mode. Further, through extensive experiments across several settings, we demonstrated the effectiveness of WAG as a tool to improve performance across diverse parameter-centric LLM applications.







\section*{Acknowledgments}
We gratefully acknowledge the support of the Google TPU Builders Program, which provided support and access to computational resources, and Google Tunix library for post-training functionality and flexibility to enable this work.

\bibliographystyle{plainnat}
\bibliography{refs}

\clearpage
\appendix
\section*{Appendix}
Here, we provide detailed related work in Appendix \ref{app:related_works}. Next, we provide the proofs of the theoretical analysis in Appendix \ref{app:proofs}. In Appendix \ref{app:perf_metrics}, we provide details of the performance metrics used in the experiments. Appendices \ref{app:qualitative} and \ref{app:quantitaive_collapse} provide the qualitative generation examples and additional empirical analysis referenced in Section~\ref{sec:emp_analysis}. Moreover, Appendix \ref{app:exp_setup} discusses the implementation details of the experiments. Finally, we provide configuration details for reproducibility of the experiments in Appendix \ref{app:code}.

\section{Related Works}
\label{app:related_works}
In this section, we review the literature on three aspects. We will discuss how parameter importance has been quantified. Next, we examine how these measures are operationalized as parameter influence estimates in downstream applications. Finally, we explore failure modes in LLMs, specifically model collapse, and observe that existing works do not examine collapse induced by masking the very parameters deemed important.

\textbf{General Parameter Importance.} 
Historically, numerous studies have investigated how to quantify parameter importance. Existing approaches primarily characterize parameter importance from three perspectives: (1) magnitude-based importance, (2) gradient-based sensitivity, and (3) curvature-based measures. The most naive magnitude-based method is utilizing parameter values themselves, under the intuition that larger weights exert a greater direct influence on the transformation of input signals and therefore may be more important \cite{hanson1988comparing, gale2019state}. To obtain a more informative estimate of importance, subsequent works incorporate activation magnitudes, which capture the effective contribution of a parameter to the downstream computation for a given input \cite{han2015learning, sun2024simple}. Beyond magnitude-based criteria, parameter importance has also been estimated through the sensitivity of the loss function. For example, approximation of the change in loss resulting from parameter removal using a first-order Taylor expansion \cite{molchanov2017pruning}. Curvature-based approaches further account for second-order information. Optimal Brain Damage estimates the increase in loss caused by removing a parameter through a second-order Taylor expansion combined with a diagonal approximation of the Hessian \cite{lecun1990optimal}. In contrast, Optimal Brain Surgeon provides a more accurate second-order estimate by utilizing the full Hessian without imposing a diagonal approximation \cite{hassibi1993second}. More recently, Optimal Brain Apoptosis extends this line of work by introducing a computationally more efficient framework while similarly avoiding the diagonal Hessian assumption \cite{sun2025optimal}. Recent quantization methods such as GPTQ \cite{frantar2022gptq} implicitly encode parameter importance through a second-order quadratic optimization objective. In these methods, a parameter's importance is reflected by how its perturbation affects the loss under the local curvature of the optimization landscape.

\textbf{Applications of Estimating Parameter Influence. }
The parameter importance estimates are widely operationalized in pruning methods, which remove parameters or structured groups based on their estimated contribution to model performance. Recent approaches for LLMs include 
magnitude-based methods like Wanda \cite{sun2024simple} which leverage activation-scaled magnitude scores to directly rank weights for removal. On the other hand, magnitude alone also serves as an importance metric where lower magnitude parameters are pruned \cite{han2015learning}. More recent work further explores gradient-based criteria and mask refinement strategies to improve pruning decisions under limited calibration data \cite{das2023beyond, yang2025wanda++, zimmer2025sparseswaps}. Recent work has also shown that Wanda pruning can improve test-time reasoning performance in LLMs, even in comparison to the base unpruned model \cite{monjur2026revisiting}.

These principles have been extended to layer-wise sparsity allocation, where methods estimate the relative contribution of transformer blocks to guide non-uniform pruning ratios across layers. AlphaPruning \cite{lu2024alphapruning} allocates sparsity using spectral characteristics of layer weights to identify well-conditioned layers, while OWL \cite{yin2023outlier} leverages activation outlier statistics to estimate layer sensitivity and preserve critical transformer blocks. On the other hand, LayerIF \cite{askari2026layerif} takes a data-driven approach, computing layer-wise influence scores by isolating layer gradients and Hessians while keeping the validation and training data constant. 

\looseness-1 Beyond pruning and sparsity allocation, influence-based \cite{chhabra2025oga} layer quality estimates have been applied to continual learning, where the central challenge is preventing catastrophic forgetting of prior tasks when learning new ones. Regularization-based methods, Elastic Weight Consolidation (EWC) \cite{kirkpatrick2017overcoming}, Synaptic Intelligence (SI) \cite{zenke2017continual}, and Memory Aware Synapses (MAS) \cite{aljundi2018memory} address this by penalizing changes to parameters deemed important for previous tasks. Recent work \cite{gao2025defying} provides a unified theoretical account of this family by formally defining parameter importance via influence functions.

Influence-based layer quality estimation has also been applied to expert allocation in Mixture-of-Experts (MoE) architectures. Uniform allocation of LoRA experts across transformer layers has been questioned by recent work showing that expert redundancy and routing overfitting degrade performance~\cite{gao2025mola}. AlphaLoRA~\cite{qing2024alphalora} addresses this using spectral layer quality metrics to inform allocation. LayerIF~\cite{askari2026layerif} instead derives task-specific allocation budgets from layer-wise influence scores, dynamically varying expert counts across layers.

Parameter influence is also used in the task of mixed-precision quantization, where parameter groups with the most saliency are preserved in high precision and the others in low precision. Recent studies, PB-LLM \cite{shang2024pbllm} and SliM-LLM \cite{huang2025slimllm}, employ magnitude- and Hessian-based criteria as a measure of the saliency of the groups.

Moreover, gradient-based methods like Salient Layers Editing Model \cite{mishra-etal-2024-correcting} and Gradient Layer Analysis \cite{datta2026golden} are used for important editing layer identification in knowledge editing.

\textbf{Identifying Failure Modes in LLMs. }
Failure due to model collapse in LLMs broadly manifests in two forms: collapse arising from corrupted training data, and collapse arising from architectural modification. Recent work has characterized the phenomenon of model collapse arising from recursive training on synthetic data \cite{shumailov2023curse}. Critically, it was shown that access to original human-generated data is essential to prevent collapse. Recent work \cite{wang2024language} shows that recursive training on self-generated text introduces non-negative errors that accumulate across generations, leading to irreversible model collapse and deviation from the original data distribution and eventual degradation in both fluency and downstream task performance.

At the architectural level, a recent study \cite{shi2026understanding} found that collapse under layer pruning is a threshold phenomenon: performance remains stable until pruning reaches the \textit{Silent Phase}, the early-to-middle layers responsible for building representational scaffolding, after which collapse is abrupt and irreversible even under supervised fine-tuning. Recent observations suggest that the collapse threshold is determined not by total parameter count but by a sparse, irregularly distributed set of critical parameters \cite{reda2025many}. However, these works do not characterize collapse at the level of individual parameters or identify the minimum number of scattered weights whose targeted removal suffices to induce it.

\section{Proofs of Theoretical Analysis}
\label{app:proofs}

We restate each theorem from Subsection~\ref{sec:theoretical_analysis} and provide its full proof. Each restated statement carries the same numbering as in the main text. Throughout, $\ell(z;\theta)$ denotes the per-example loss of a model with parameter vector $\theta = (\theta_1,\ldots,\theta_d) \in \mathbb{R}^d$ evaluated at example $z$, and the \textsc{WAG} of coordinate $i$ is defined as
\begin{equation*}
\mathrm{WAG}_i \;=\; -\,\theta_i \frac{\partial \ell(z;\theta)}{\partial \theta_i}.
\end{equation*}
All perturbations are evaluated to first order, with higher-order terms collected in
$\mathcal{O}(\cdot)$.

\subsection{Proof of Theorem~\ref{thm:log_gradient}}
\logGradient*
\begin{proof}
Since $u_i = \log |\theta_i|,$ we have $\theta_i = \operatorname{sign}(\theta_i)e^{u_i}.$ Applying the chain rule gives:
\begin{equation*}
\frac{\partial \ell}{\partial u_i}
=
\frac{\partial \ell}{\partial \theta_i}
\frac{\partial \theta_i}{\partial u_i}.
\end{equation*}
Moreover, since $\frac{\partial \theta_i}{\partial u_i} = \theta_i$, it follows that:
\begin{equation*}
\frac{\partial \ell}{\partial u_i}
=
\theta_i
\frac{\partial \ell}{\partial \theta_i}.
\end{equation*}
And using the definition of WAG yields:
\begin{equation*}
\mathrm{WAG}_i
=
-
\frac{\partial \ell}{\partial u_i}. \qedhere
\end{equation*}
\end{proof}

\subsection{Proof of Theorem~\ref{thm:multiplicative}}
\multiplicativeSensitivity*
\begin{proof}
Using first-order Taylor expansion around $\theta$,
\begin{equation*}
\ell(z;\theta')
=
\ell(z;\theta)
+
\sum_{i=1}^d
\frac{\partial \ell(z;\theta)}{\partial \theta_i}
(\theta_i' - \theta_i)
+
\mathcal{O}(\|\theta'-\theta\|^2).
\end{equation*}
As $\theta_i' - \theta_i = \epsilon_i \cdot \theta_i$, we obtain:
\begin{equation*}
\ell(z;\theta') - \ell(z;\theta)
=
\sum_{i=1}^d
\epsilon_i
\theta_i
\frac{\partial \ell(z;\theta)}{\partial \theta_i}
+
\mathcal{O}(\|\epsilon\|^2).
\end{equation*}
Substituting the definition of WAG ties up the proof:
\begin{equation*}
\ell(z;\theta') - \ell(z;\theta)
=
-
\sum_{i=1}^d
\epsilon_i \mathrm{WAG}_i
+
\mathcal{O}(\|\epsilon\|^2). \qedhere
\end{equation*}
\end{proof}

\subsection{Proof of Theorem~\ref{thm:scale_balance}}
\label{app:scalebalance_proof}
We show that the WAG vector identifies the dominant directions governing local scale stability. In particular, the magnitude of the WAG vector determines the largest possible first-order increase in loss achievable through multiplicative perturbations. Consequently, parameters with large absolute WAG scores correspond to the coordinates that contribute most strongly to maintaining local scale balance.

\looseness-1 Modern transformer-based LLMs contain billions of parameters whose magnitudes vary substantially across layers, attention blocks, and MLP blocks \cite{an2025systematic}. Despite this heterogeneity, trained models exhibit remarkably stable behavior. A key reason is that a network's behavior depends on the relative balance of these scales, so the loss responds to how parameter scales are coordinated across the model \cite{dinh2017sharp, ba2016layer}. This observation suggests that optimization settles into a coordinated balance of parameter scales. Prior measures probe this geometry by perturbing weights additively~\cite{foret2021sharpness}; in contrast, WAG measures the sensitivity of the loss to relative changes in these scales, making it a natural tool for studying this \textit{balance}. Therefore, we next show that the WAG vector identifies the dominant directions governing local scale stability.

\scaleBalance*
\begin{proof}
By Theorem~\ref{thm:multiplicative},
\begin{equation*}
\ell(z;\theta') - \ell(z;\theta) = -\langle \epsilon,w\rangle + \mathcal{O}(\|\epsilon\|_2^2).
\end{equation*}
Fix a perturbation budget $\|\epsilon\|_2=r$. Applying the Cauchy--Schwarz inequality yields:
\begin{equation*}
-\langle \epsilon,w\rangle \le \|\epsilon\|_2\|w\|_2 = r\|w\|_2.
\end{equation*}
Equality holds precisely when $\epsilon$ is anti-parallel to $w$, namely:
\begin{equation*}
\epsilon^\star = -\frac{r\,w}{\|w\|_2}.
\end{equation*}
Substituting this perturbation into the first-order expansion gives us:
\begin{equation*}
\ell(z;\theta') - \ell(z;\theta) = r\|w\|_2 + \mathcal{O}(r^2).
\end{equation*}
\looseness-1Therefore, no multiplicative perturbation with the same norm can produce a larger first-order increase in loss. Since Theorem~\ref{thm:log_gradient} showed that WAG is the gradient in logarithmic parameter coordinates, it follows that the WAG vector defines the direction of greatest local sensitivity in log-parameter space.
\end{proof}

Theorem~\ref{thm:scale_balance} provides a geometric interpretation of WAG. In log-parameter space, the WAG vector identifies the unique direction along which multiplicative perturbations produce the largest first-order increase in loss. The norm $\|w\|_2$ quantifies the magnitude of this worst-case sensitivity, while the coordinates with the largest absolute WAG values contribute most strongly to it. 

Combined with Theorem~\ref{thm:log_gradient}, this result suggests that trained transformers (i.e., LLMs) can be viewed as \textit{scale-balanced} systems. Optimization drives the network toward configurations in which many parameter scales cooperate to maintain low loss. The WAG vector captures the residual sensitivity of this equilibrium and identifies the directions along which the balance remains most fragile. 
An immediate consequence is that not all parameters contribute equally to maintaining this equilibrium. Parameters with large absolute WAG scores dominate the most sensitive log-scale direction and therefore account for a disproportionate fraction of the model's local scale stability. Perturbing or masking these parameters removes the coordinates that contribute most strongly to preserving the learned \textit{balance} of \textit{scales}. 

This perspective provides a theoretical explanation for the empirical phenomena and observations we provide later in the paper. When parameters are ranked according to $|\mathrm{WAG}_i|$, a subset of coordinates often emerges whose modification causes significant degradation in model performance. As described via Theorem~\ref{thm:scale_balance}, these parameters correspond precisely to the coordinates that dominate the model's most sensitive log-scale direction. In sum, masking them therefore removes a substantial portion of the network's scale-balancing structure, which results in significant performance degradation and even model collapse.

\section{Performance Metrics}
\label{app:perf_metrics}
Here we give the details of the metrics used to evaluate the results of the experiments involving the applications of \textsc{WAG} presented in this paper.

\subsection{Empirical Analysis}
\label{app:perf_metrics_empirical}
We adopt the evaluation protocols of OLMES \cite{gu2025olmesstandardlanguagemodel}, LM Evaluation Harness \cite{eval-harness}, and InspectAI \cite{abbas2025inspectevals} for coding, mathematical, and commonsense reasoning tasks, respectively for our experiments.

\textbf{\textit{Pass@1}:} \noindent \textit{Pass@1} is used to evaluate code generation performance on the \textit{HumanEval} and \textit{HumanEval+} benchmarks. It measures the proportion of tasks for which the model's greedy generated solution passes all corresponding unit tests.

\textbf{\textit{Exact Match}:} \noindent \looseness-1 \textit{Exact Match} is used to evaluate mathematical reasoning performance on the \textit{GSM8K} benchmark and commonsense reasoning performance on the \textit{HellaSwag} benchmark. It measures the proportion of examples for which the model's final generated answer exactly matches the ground-truth answer.

\subsection{Application A: \textsc{WAG} for Mixture-of-Experts Allocation}
\looseness-1 We evaluate the expert allocations following the LayerIF \cite{askari2026layerif} evaluation protocol.

\textbf{\textit{Zero-shot Accuracy}:} \textit{Zero-shot Accuracy} measures classification and question-answering performance on \textit{MRPC}, \textit{CoLA}, and \textit{ScienceQA}, computed as the proportion of examples where the label extracted from the model's free-form response via pattern matching exactly matches the ground truth.

\subsection{Application B: \textsc{WAG} for Targeted Unlearning}
We adopt the evaluation of the unlearning task from the \textit{TOFU} benchmark \cite{maini2024tofu}. Each QA example is probed with its ground-truth answer, a \emph{paraphrased} (semantically equivalent) answer, and a set of \emph{perturbed} (factually wrong) answers. For a target answer with length-normalized negative log-likelihood (per-token cross-entropy) $\ell$ under the model, its length-normalized probability is $p=\exp(-\ell)$. The benchmark partitions data into a \textit{Forget} set, a \textit{Retain} set, and two general-knowledge probes (\textit{Real Authors}, \textit{Real World}).

\textbf{\textit{Model Utility}:} \noindent The overall linguistic and retention capability preserved after unlearning, computed as the harmonic mean of nine non-forget metrics: \textit{Probability, ROUGE, Truth Ratio} for each \textit{Retain, Real Authors, Real World} type of knowledge probe. For each task, \emph{Probability} is the mean ground-truth answer probability (normalized against the perturbed answers for the Real Authors / Real World probes); \emph{ROUGE} is the ROUGE-L recall
(with stemming) between the model's greedily generated answer and the ground truth; and \emph{Truth Ratio} rewards preferring the correct answer over perturbations.

\textbf{\textit{Forget Truth Ratio (TR)}:} \noindent Let $p_{\text{para}}=\exp(-\ell_{\text{para}})$ be the probability of the paraphrased answer and $\bar p_{\text{pert}}$ the mean probability over
the perturbed answers, and define $R=\bar p_{\text{pert}}/p_{\text{para}}$. On the Forget set we report $\frac{1}{N}\sum \min(R,\,1/R)\in(0,1]$. A higher value means the perturbed and correct answers are about equally (un)likely, indicating the targeted fact has been forgotten.

\textbf{\textit{KS Forget}:} \noindent We form the empirical distribution of per-example forget truth ratios for the unlearned model with respect to the FFT reference model, and run a two-sample Kolmogorov--Smirnov test; \textit{KS Forget} is the resulting KS statistic $D$ (the supremum distance between the two empirical Cumulative Distribution Functions (CDFs)). In our setup, the reference is the fine-tuned model \emph{before} unlearning (FFT), so a larger $D$ indicates that the unlearned model's forget-set behavior has diverged further from the pre-unlearning model i.e., stronger forgetting. As the FFT model itself is the reference model, the metric is not reported for the FFT model.

\subsection{Application C: \textsc{WAG} for Mixed Precision Quantization}
We employ the same performance metrics defined in Appendix \ref{app:perf_metrics_empirical} for this application.

\subsection{Application D: \textsc{WAG} for Knowledge Editing}
We evaluate knowledge editing performance using five metrics provided by the EasyEdit framework \cite{wang2024easyediteasytouseknowledgeediting}: \textit{Rewrite Accuracy}, \textit{Rephrase Accuracy}, \textit{Locality}, \textit{Portability}, and \textit{Fluency}. Since not all datasets support every evaluation metric \cite{wang2024easyediteasytouseknowledgeediting, zhang2024comprehensive}, results are reported only for the metrics applicable to each dataset. Higher values indicate better editing performance for all metrics. We additionally report an \textit{Overall} score, computed as the average of the normalized metric values.

\textbf{\textit{Rewrite Accuracy}:} \textit{Rewrite Accuracy} measures whether the edited knowledge is correctly retrieved when responding to the original query targeted by the edit.

\textbf{\textit{Rephrase Accuracy}:} \textit{Rephrase Accuracy} evaluates whether the model can recall the edited knowledge when presented with semantically equivalent rephrasings of the original query.

\textbf{\textit{Locality}:} \textit{Locality} measures the extent to which the edit preserves the model's behavior on unrelated queries, thereby quantifying unintended side effects introduced by the edit \cite{meng2022locatingRome, wang2024easyediteasytouseknowledgeediting, yao2023editinglargelanguagemodels, zhang2024comprehensive}. 

\textbf{\textit{Portability}:} \textit{Portability} assesses the model's ability to transfer the edited knowledge to related downstream reasoning tasks \cite{yao2023editinglargelanguagemodels}.

\textbf{\textit{Fluency}:} \textit{Fluency} evaluates the linguistic quality and naturalness of the generated responses.

\section{Empirical Analysis}
\label{app:empirical}
In this section, we present the qualitative and quantitative analysis of model behavior due to masking \textsc{WAG}-identified parameters.

\subsection{Quantitative Examples of WAG-Induced Collapse}
\label{app:quantitaive_collapse}
Masking up to just $\leq$ 0.0002\% of $|\text{WAG}|$ identified parameters causes model collapse or the onset of model collapse across different tasks and models. However, masking the same number of parameters identified by weight or gradient doesn't cause consistent model collapse. Therefore, WAG is able to identify important parameters consistently.

\begin{figure*}[!t]
    \centering
    \includegraphics[width=0.90\textwidth]{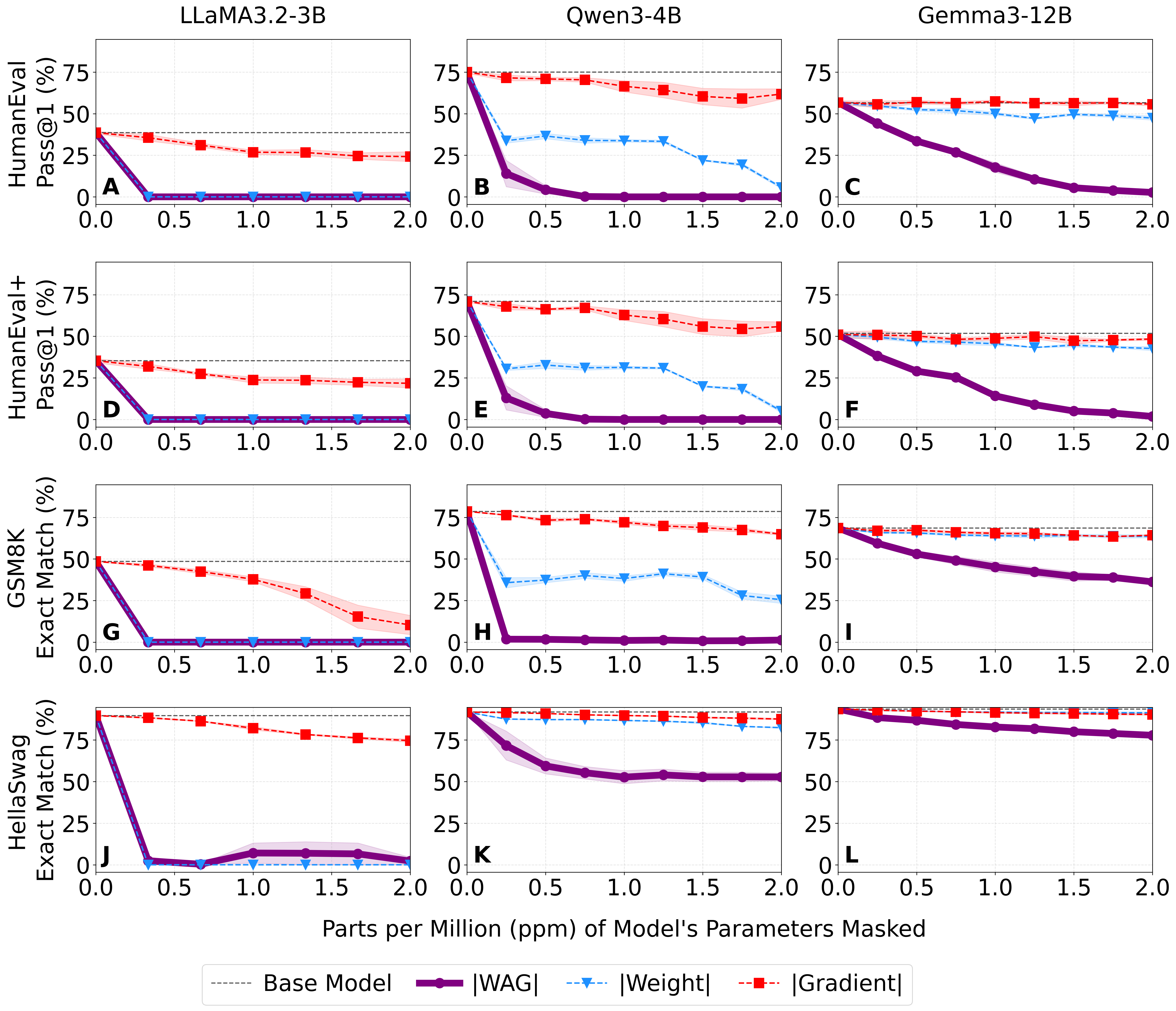}
    \caption{Illustration of model collapse caused by masking just $\leq 2$\ parts per million (i.e., $\leq 0.0002\%$) parameters across several models: LLaMA3.2-3B (A, D, G and J), Qwen3-4B (B, E, H and K), and Gemma3-12B (C, F, I and L). Each model is assessed on four datasets: \textit{HumanEval} (A, B and C), \textit{HumanEval+} (D, E and F), \textit{GSM8K} (G, H and I), and \textit{HellaSwag} (J, K, and L). Performance is reported as \textit{Pass@1} and \textit{exact match} accuracy as a function of the number of masked parameters. Parameters are selected based on having the largest-magnitude gradient $\rightarrow$ |Gradient|, weight $\rightarrow$ |Weight|, and \textsc{WAG} $\rightarrow$ $|\text{WAG}|$ values. Masking the largest-magnitude \textsc{WAG} identified parameters leads to rapid model collapse with a comparatively tiny portion of masked parameters.}
    \label{fig:abs}
    \vspace{-6mm}
\end{figure*}

\subsection{Qualitative Examples of WAG-Induced Collapse}
\label{app:qualitative}
We provide the qualitative generation examples referenced in Section~\ref{sec:emp_analysis}, showing verbatim model outputs before and after masking the WAG-identified parameters, indicating qualitative collapse after masking only 0.0001\% of the parameters.

\begin{figure}[t]
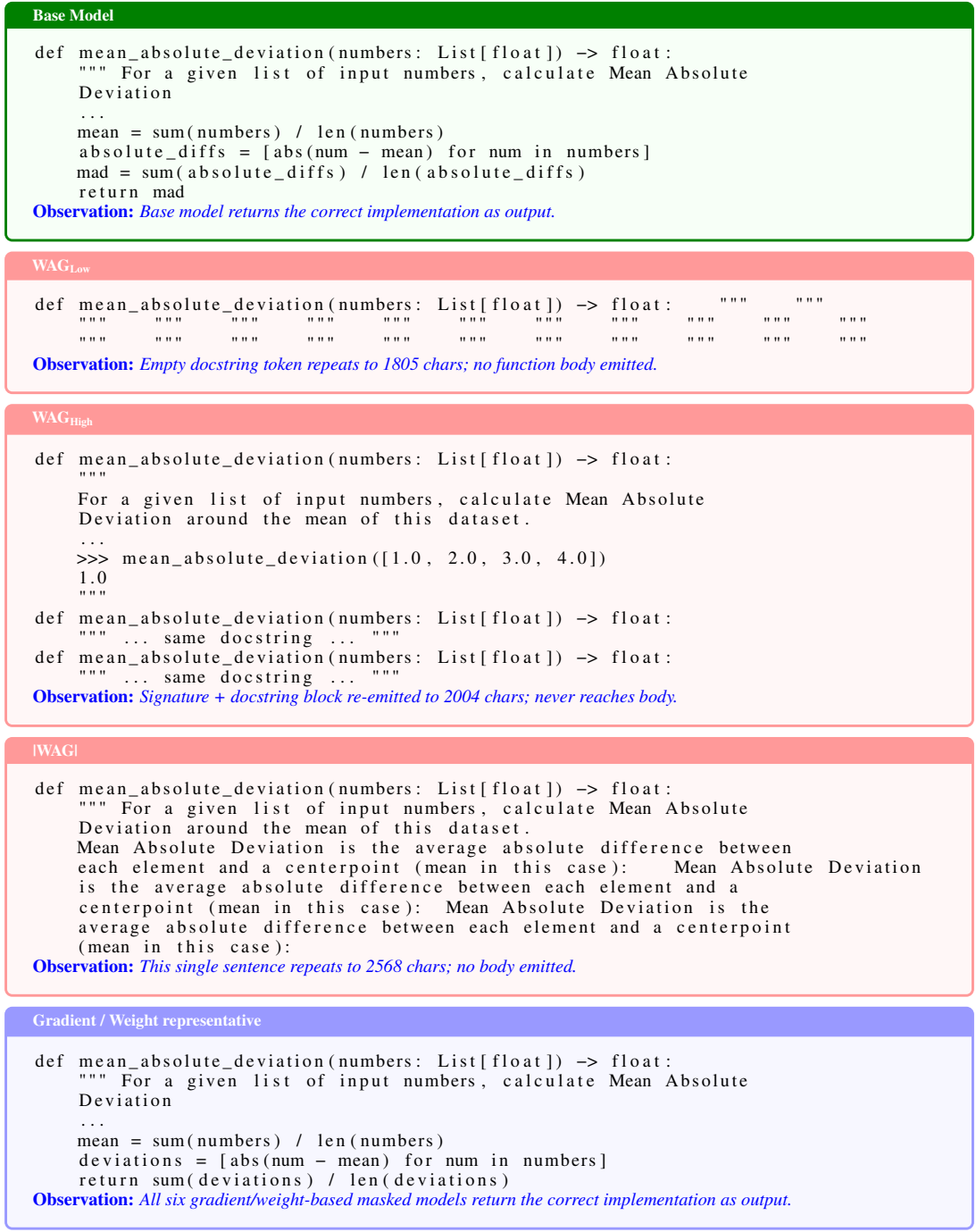

\centering

\newsavebox{\humanevalbox}
\begin{lrbox}{\humanevalbox}
\begin{minipage}{1.14\textwidth}


\begin{tcblisting}{
  % title={Base Model (pass@1 76.2\%)},
  title={Base Model},
  colback=green!3, colframe=green!50!black,
  fonttitle=\bfseries\footnotesize,
  listing only,
  listing options={escapeinside={(*@}{@*)}},
  top=0pt, bottom=0pt,
  % listing options={basicstyle=\ttfamily\scriptsize,breaklines=true,columns=fullflexible},
}
def mean_absolute_deviation(numbers: List[float]) -> float:
    """ For a given list of input numbers, calculate Mean Absolute
    Deviation
    ...
    mean = sum(numbers) / len(numbers)
    absolute_diffs = [abs(num - mean) for num in numbers]
    mad = sum(absolute_diffs) / len(absolute_diffs)
    return mad
(*@\textcolor{blue}{\textbf{Observation:} \textit{Base model returns the correct implementation as output.}}@*)
\end{tcblisting}

\begin{tcblisting}{
  % title={WAG Low (pass@1 0.0\%)},
  title={WAG$_\text{Low}$},
  colback=red!3, colframe=red!40,
  fonttitle=\bfseries\footnotesize,
  listing only,
  listing options={escapeinside={(*@}{@*)}},
  top=0pt, bottom=0pt,
  % listing options={basicstyle=\ttfamily\scriptsize,breaklines=true,columns=fullflexible}
}
def mean_absolute_deviation(numbers: List[float]) -> float:    """    """
    """    """    """    """    """    """    """    """    """    """    """
    """    """    """    """    """    """    """    """    """    """    """
(*@\textcolor{blue}{\textbf{Observation:} \textit{Empty docstring token repeats to 1805 chars; no function body emitted.}}@*)
\end{tcblisting}

\begin{tcblisting}{
  % title={WAG High (pass@1 18.9\%)},
  title={WAG$_\text{High}$},
  colback=red!3, colframe=red!40,
  fonttitle=\bfseries\footnotesize,
  listing only,
  listing options={escapeinside={(*@}{@*)}},
  top=0pt, bottom=0pt,
}
def mean_absolute_deviation(numbers: List[float]) -> float:
    """
    For a given list of input numbers, calculate Mean Absolute
    Deviation around the mean of this dataset.
    ...
    >>> mean_absolute_deviation([1.0, 2.0, 3.0, 4.0])
    1.0
    """
def mean_absolute_deviation(numbers: List[float]) -> float:
    """ ... same docstring ... """
def mean_absolute_deviation(numbers: List[float]) -> float:
    """ ... same docstring ... """
(*@\textcolor{blue}{\textbf{Observation:} \textit{Signature + docstring block re-emitted to 2004 chars; never reaches body.}}@*)
\end{tcblisting}

\begin{tcblisting}{
  % title={abs-WAG High (pass@1 0.0\%)},
  title={|WAG|},
  colback=red!3, colframe=red!40,
  fonttitle=\bfseries\footnotesize,
  listing only,
  listing options={escapeinside={(*@}{@*)}},
  top=0pt, bottom=0pt,
  % listing options={basicstyle=\ttfamily\scriptsize,breaklines=true,columns=fullflexible}
}
def mean_absolute_deviation(numbers: List[float]) -> float:
    """ For a given list of input numbers, calculate Mean Absolute
    Deviation around the mean of this dataset.
    Mean Absolute Deviation is the average absolute difference between
    each element and a centerpoint (mean in this case):    Mean Absolute Deviation
    is the average absolute difference between each element and a
    centerpoint (mean in this case):  Mean Absolute Deviation is the
    average absolute difference between each element and a centerpoint
    (mean in this case):
(*@\textcolor{blue}{\textbf{Observation:} \textit{This single sentence repeats to 2568 chars; no body emitted.}}@*)
\end{tcblisting}

\begin{tcblisting}{
  % title={Gradient / Weight representative (pass@1 34-73\%)},
  title={Gradient / Weight representative},
  colback=blue!3, colframe=blue!40,
  fonttitle=\bfseries\footnotesize,
  listing only,
  listing options={escapeinside={(*@}{@*)}},
  top=0pt, bottom=0pt,
  % listing options={basicstyle=\ttfamily\scriptsize,breaklines=true,columns=fullflexible}
}
def mean_absolute_deviation(numbers: List[float]) -> float:
    """ For a given list of input numbers, calculate Mean Absolute
    Deviation
    ...
    mean = sum(numbers) / len(numbers)
    deviations = [abs(num - mean) for num in numbers]
    return sum(deviations) / len(deviations)
(*@\textcolor{blue}{\textbf{Observation:} \textit{All six gradient/weight-based masked models return the correct implementation as output.}}@*)
\end{tcblisting}

\end{minipage}
\end{lrbox}

\resizebox{0.87\textwidth}{!}{\usebox{\humanevalbox}}\vspace{-1.5mm}

\caption{Qualitative example of model failure caused by masking 0.0001\% of parameters identified by \textsc{WAG}. Verbatim Qwen3-4B generations on \textit{HumanEval} at 4,000 masked parameters indicate that the WAG$_\text{Low}$ variant emits only empty docstrings; WAG$_\text{High}$ repeats the signature and docstring without a body; $|\text{WAG}|$ loops on a single sentence. All six gradient- and weight-masked variants return the canonical solution.}
\label{fig:qual-humaneval}
\vspace{-6mm}
\end{figure}

\begin{figure}[t]
\centering\footnotesize
\begin{tcolorbox}[colback=gray!4,colframe=gray!40,boxrule=0.3pt,left=2pt,right=2pt,top=2pt,bottom=2pt]
\textbf{Prompt:}\emph{``While walking with his 3 younger siblings, Greg found \$20 and split it
equally. How much did each get?''}
\end{tcolorbox}

\begin{tcblisting}{
  % title={Base Model (Exact match 78.5\%)},
  title={Base Model},
  colback=green!3, colframe=green!50!black,
  fonttitle=\bfseries\footnotesize,
  listing only,
  listing options={escapeinside={(*@}{@*)}},
  top=0pt, bottom=0pt,
  % listing options={basicstyle=\ttfamily\scriptsize,breaklines=true,columns=fullflexible}
}
There were 1 + 3 = 4 of them who found the money.
Each got 20 / 4 = $5.   ####  5
(*@\textcolor{blue}{\textbf{Observation:} \textit{Correct answer.}}@*)
\end{tcblisting}

\begin{tcblisting}{
  % title={WAG Low (Exact match 32.4\%)},
  title={WAG$_\text{Low}$},
  colback=red!3, colframe=red!40,
  fonttitle=\bfseries\footnotesize,
  listing only,
  listing options={escapeinside={(*@}{@*)}},
  top=0pt, bottom=0pt,
  % listing options={basicstyle=\ttfamily\scriptsize,breaklines=true,columns=fullflexible}
}
Greg found $20 and decided to split it equally among his 3 siblings,
so each of them got $20 / 3 = $6 and 2/3.   ####  6
(*@\textcolor{blue}{\textbf{Observation:} \textit{Fluent, but omits Greg and divides by 3 instead of 4.}}@*)
\end{tcblisting}

\begin{tcblisting}{
  % title={WAG High (Exact match 1.0\%)},
  title={WAG$_\text{High}$},
  colback=red!3, colframe=red!40,
  fonttitle=\bfseries\footnotesize,
  listing only,
  listing options={escapeinside={(*@}{@*)}},
  top=0pt, bottom=0pt,
  % listing options={basicstyle=\ttfamily\scriptsize,breaklines=true,columns=fullflexible}
}
To be fair to his 3 1 1 1 1 1 1 1 1 1 1 1 1 1 1 1 1 1 1 1 1 1 1 1 1 1
1 1 1 1 1 1 1 1 1 1 1 1 1 1 1 1 1 1 1 1 1 1 1 1 1 1 1 1 1 1 1 1 1 ...
(*@\textcolor{blue}{\textbf{Observation:} \textit{Token "1" repeats to the length cap.}}@*)
\end{tcblisting}

\begin{tcblisting}{
  % title={abs-WAG High (Exact match 1.1\%)},
  title={|WAG|},
  colback=red!3, colframe=red!40,
  fonttitle=\bfseries\footnotesize,
  listing only,
  listing options={escapeinside={(*@}{@*)}},
  top=0pt, bottom=0pt,
  % listing options={basicstyle=\ttfamily\scriptsize,breaklines=true,columns=fullflexible}
}
GregGregGregGregGregGregGregGregGregGregGregGregGregGregGregGregGreg
GregGregGregGregGregGregGregGregGregGregGregGregGregGregGregGreg...
(*@\textcolor{blue}{\textbf{Observation:} \textit{Token "Greg" repeats to 2048 chars; the extracted answer is the same loop.}}@*)
\end{tcblisting}

\begin{tcblisting}{
  % title={Gradient / Weight representative (Exact match 22--71\%)},
  title={Gradient / Weight representative },
  colback=blue!3, colframe=blue!40,
  fonttitle=\bfseries\footnotesize,
  listing only,
  listing options={escapeinside={(*@}{@*)}},
  top=0pt, bottom=0pt,
  % listing options={basicstyle=\ttfamily\scriptsize,breaklines=true,columns=fullflexible}
}
There are 4 of them because 1 + 3 = 4.
Each of them gets $5 because 20 / 4 = 5.   ####  5  
(*@\textcolor{blue}{\textbf{Observation:} \textit{All six gradient/weight variants recover the answer.}}@*)
\end{tcblisting}\vspace{-3mm}
\caption{Qualitative example of model failure caused by masking 0.0002\% \textsc{WAG}-identified parameters. Verbatim Qwen3-4B generations on \textit{GSM8K} at 8,000 masked parameters indicate that WAG$_\text{High}$ and $|\text{WAG}|$ masking produces pure token repetition (``1\ldots'', ``Greg\ldots''); WAG$_\text{Low}$ masking stays fluent but miscounts. All six gradient- and weight-masked variants solve the problem.}
\label{fig:qual-gsm8k}
\vspace{-5mm}
\end{figure}

\subsection{\textsc{WAG} selected parameter distribution across layers.}
\label{app:wag_distribution}
 To examine the layer-wise distribution of the 0.0002\% parameters selected by WAG, Figure \ref{fig:wag_distribution_all} shows the selected parameters across layers for WAG$_\text{Low}$, WAG$_\text{High}$, and $|\textsc{WAG}|$ across models and datasets. Across all datasets and models, the selected parameters are predominantly concentrated in the final layer, while constituting only a very small fraction of its parameters.

\begin{figure*}[t]
    \centering
    \includegraphics[width=0.90\textwidth]{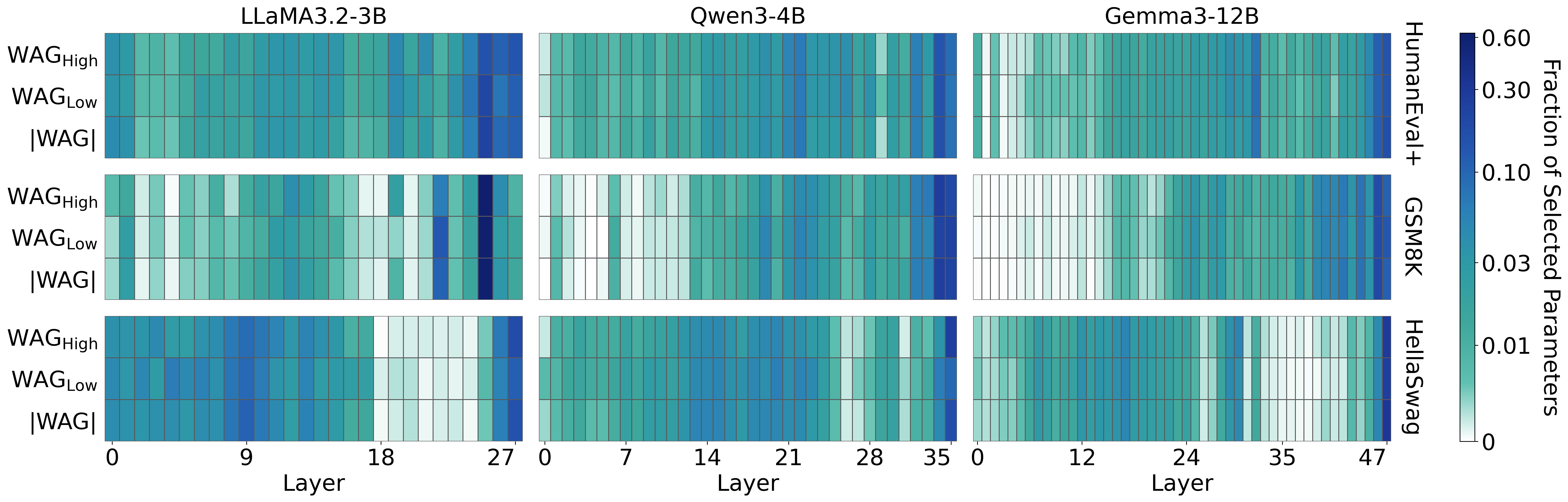}\vspace{-0mm}
    \caption{\looseness-1 Distribution of 2 parts per million or 0.0002\% selected parameters by WAG$_\text{Low}$, WAG$_\text{High}$, and $|\textsc{WAG}|$ across layers of LLaMA3.2-3B, Qwen3-4B, and Gemma3-12B for \textit{HumanEval+}, \textit{GSM8K} and \textit{HellaSwag}.}
    \label{fig:wag_distribution_all}\vspace{-4mm}
\end{figure*}

\subsection{\textsc{WAG} score distribution across layers.}
\label{app:wag_score_distribution}
\looseness-1 To complement the parameter-selection view in Appendix~\ref{app:wag_distribution}, Figure~\ref{fig:wag_score_all} shows the layer-wise \emph{average} \textsc{WAG} score, computed separately over the positive elements (WAG$_+$), the negative elements (WAG$_-$), all elements (WAG), and the magnitude ($|\textsc{WAG}|$) of all parameters for each layer. Two observations hold across all models and datasets. First, the largest score magnitudes for WAG$_+$, WAG$_-$, and $|\textsc{WAG}|$ are concentrated in the early-to-mid layers. The concentration depth is dataset-dependent (mid layers for \textit{HumanEval+} and \textit{GSM8K}, early layers for \textit{HellaSwag}) but, for a given dataset, sits at the same relative layer depth in LLaMA3.2-3B, Qwen3-4B, and Gemma3-12B. Second, the average WAG is close to zero in every layer: positive and negative contributions are evenly balanced within each layer.

\begin{figure*}[t]
    \centering
    \includegraphics[width=0.90\textwidth]{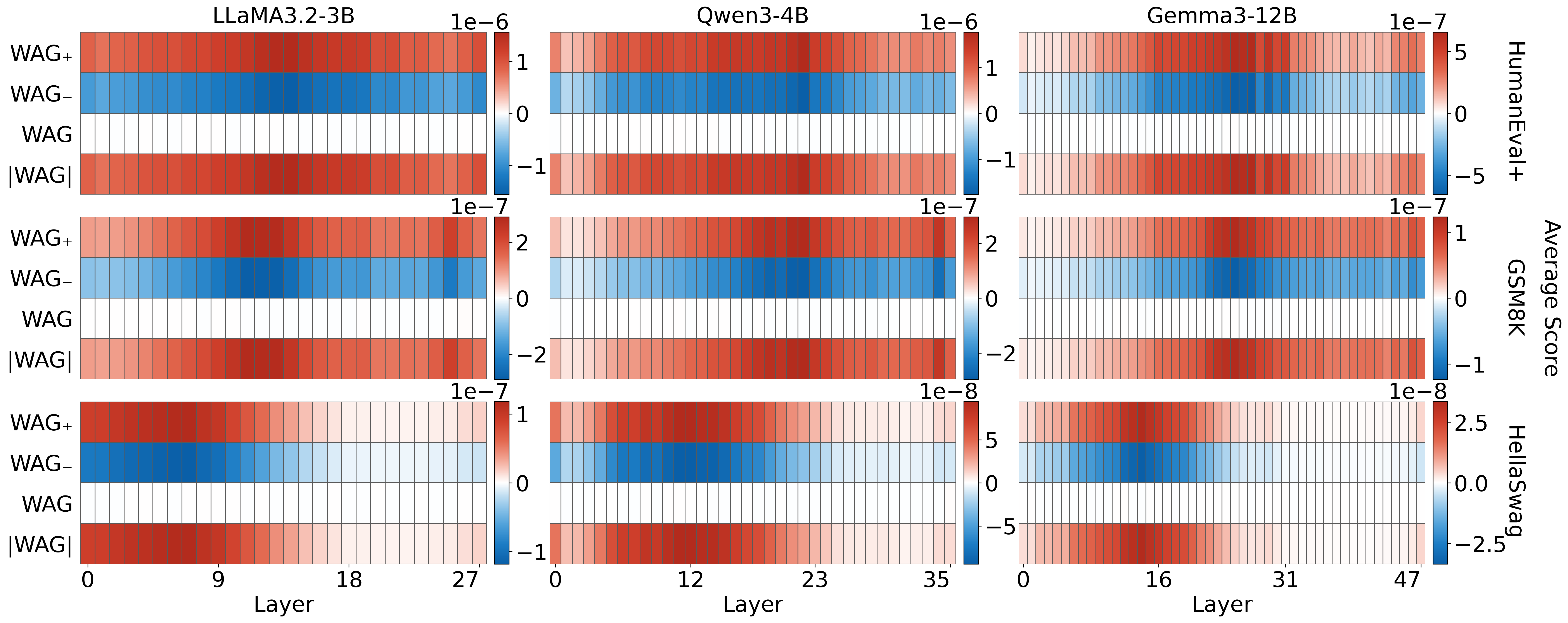}\vspace{-2mm}
    \caption{Layer-wise average \textsc{WAG} score over positive elements (WAG$_+$), negative elements (WAG$_-$), all elements (WAG), and magnitude ($|\textsc{WAG}|$) for LLaMA3.2-3B, Qwen3-4B, and Gemma3-12B on \textit{HumanEval+}, \textit{GSM8K} and \textit{HellaSwag}. All panels share one symmetric-log colour scale.}
    \label{fig:wag_score_all}\vspace{-4mm}
\end{figure*}

\section{Experiment Setup and Implementation Details}
\label{app:exp_setup}
In this section, we will discuss the experiment setup and the implementation details of the experiments conducted to demonstrate the practical applicability of \textsc{WAG}.

\subsection{Empirical Analysis and common details}
We consider all parameters of the linear projections in the attention and MLP blocks as candidate parameters. We compute a WAG score for each candidate parameter and rank them globally across all layers and blocks. Given a masking budget $k$, we mask the $k$ parameters based on the selection criteria. To compute WAG-based parameter importance, we compute the gradients of the loss with respect to the model parameters using the evaluation set of the corresponding task. We evaluate performance on all tasks using a held-out evaluation set. For the coding tasks, we train the model on coding-related datasets and evaluate it on the \textit{HumanEval} \cite{chen2021evaluating} and \textit{HumanEval+} \cite{evalplus} datasets and report the \textit{pass@1} accuracy \cite{chen2021evaluating}, as conducted in recent work \cite{zhang2026boostinglargelanguagemodels}.  For the math and commonsense reasoning tasks, we selected the train and test split of the \textit{GSM8K} \cite{cobbe2021training} and \textit{HellaSwag} \cite{zellers2019hellaswag} datasets for training and evaluation, respectively, and report the \textit{exact match} accuracy.

\subsection{Application A: \textsc{WAG} for Mixture-of-Experts Allocation}
\label{app:expert_allocation}


We first leverage \textsc{WAG} to determine the optimal allocation of a fixed budget of $N$ parallel experts across layers. We compare our method against four baselines: AlphaLoRA~\cite{qing2024alphalora} and three fixed MoLA configurations~\cite{gao2025mola}, triangle (2 4 6 8), rectangle (5 5 5 5), and hourglass (8 2 2 8). The total expert budget is set to $N = 140$ for LLaMA3.2-3B and $N = 240$ for Gemma3-12B, and the router selects at most the top-$2$ experts per layer. We evaluate the zero-shot performance of the models with the allocated experts on three datasets: \textit{MRPC} \cite{dolan2005automatically}, \textit{CoLA} \cite{wang2018glue}, and \textit{ScienceQA} \cite{lu2022learn}.

Our allocation principle is to assign more experts to layers whose parameters are highly loss-sensitive to perturbation. Since such layers have the greatest impact on model performance, they should receive a larger share of experts to adequately capture their complexity. We therefore report the expert-allocation performance for sensitive layers identified by three \textsc{WAG}-based criteria: WAG$_\text{Low}$, WAG$_\text{High}$, and $|\text{WAG}|$. To assess how many important parameters each layer contains, treating every parameter as equally weighted in the comparison, we compute a global rank for each parameter according to its \textsc{WAG} value. Formally, let $w_p$ denote the \textsc{WAG}-based value of parameter $p$ and $\mathrm{rank}(w_p)$ its global rank among all parameters in the model, where a higher rank is assigned to parameters deemed more sensitive based on \textsc{WAG} criteria. $\mathcal{P}_\ell$ denotes the set of parameters in layer $\ell$. The score of layer $\ell$ is then defined as the mean rank of its parameters:
$
    v_i = \frac{1}{|\mathcal{P}_i|} \sum_{p \in \mathcal{P}_i}
    \mathrm{rank}(q_p),
    \label{eq:layer-score}
$
Given the layer scores $\{v_i\}_{i=1}^{m}$, where $m$ is the number of layers,
we convert them into integer expert counts following an approach similar to
LayerIF~\cite{askari2026layerif} and AlphaLoRA~\cite{qing2024alphalora}. We apply a power transformation to the scores to modulate the sharpness of the allocation. First, we normalize the (inverted) scores to the interval $[0, 1]$ via
min-max normalization:
$
    \tilde{v}_i = \frac{v_i - \min_j v_j}{\max_j v_j - \min_j v_j}.
    \label{eq:normalization}
$
We then raise the normalized scores to a hyperparameter-controlled exponent $\beta > 0$:
$
    \hat{v}_i = \tilde{v}_i^{\,\beta},
    \label{eq:power-transform}
$
, where $\beta$ controls the dispersion of the allocation, larger values of
$\beta$ amplify disparities between layers. Next, we scale the transformed scores so that the total allocation, excluding
a baseline of one expert per layer, sums to the target budget $T$:
$
    f_i = \frac{\hat{v}_i}{\sum_{j=1}^{m} \hat{v}_j} \cdot (T - m).
    \label{eq:fractional-allocation}
$
Each layer's preliminary allocation is obtained by flooring the fractional
allocation and adding the baseline expert:
$
    s_i = \lfloor f_i \rfloor + 1,
    \label{eq:floor-allocation}
$
which guarantees $s_i \geq 1$ for all layers. Due to flooring, the total
$\sum_i s_i$ may fall short of $T$. We therefore compute the residual budget
$r = T - \sum_{i} s_i$ and distribute one additional expert to each of the $r$
layers with the largest fractional remainders $f_i - \lfloor f_i \rfloor$:
$
    s_i \leftarrow s_i + 1
    \quad \text{for the top-}r \text{ layers ranked by }
    \bigl(f_i - \lfloor f_i \rfloor\bigr).
    \label{eq:remainder-allocation}
$
This procedure yields the final allocation vector
$S = [s_1, s_2, \ldots, s_m]$ satisfying $\sum_{i=1}^{m} s_i = T$ and
$s_i \geq 1$ for all $i$.

\subsection{Application B: \textsc{WAG} for Targeted Unlearning}
\label{app:application_unlearning}
\looseness-1 In this subsection, we utilize \textsc{WAG} to make unlearning more effective and more efficient. We focus on improving Gradient Difference (GradDiff) \cite{liu2022continual} by applying it selectively on top of a \textsc{WAG}-guided layer selection. Rather than following the standard GradDiff procedure of updating all model parameters during unlearning, we restrict the updates to a targeted subset comprising 75\% of the layers identified by \textsc{WAG}, while keeping the parameters of the remaining layers frozen. We conduct all experiments and evaluations on the \textit{TOFU} benchmark~\cite{maini2024tofu}, specifically on the \textit{10\% forget} dataset.

The intuition behind targeted unlearning is as follows: if unlearning is applied only to the layers containing the parameters most important with respect to the forget set, we can simultaneously (i) improve unlearning quality, since the updates concentrate on the parameters that encode the information to be removed, and (ii) reduce computational cost, since only a fraction of the parameters needs to be updated. Analogous to our expert allocation procedure in Application A \ref{sec:expert_allocation}, we identify the target layers as those containing the largest number of important parameters according to their mean rank of parameters following \textsc{WAG}-based score criteria: the lowest \textsc{WAG} values (WAG$_\text{Low}$), the highest \textsc{WAG} values (WAG$_\text{High}$), and the largest-magnitude \textsc{WAG} values ($|\textsc{WAG}|$). Formally, let $v_{\ell}$ denote the score of layer $\ell$, where the underlying \textsc{WAG} values are computed with respect to the forget set $\mathcal{D}_f$. Given a selection ratio $\rho \in (0, 1]$ (we use $\rho = 0.75$), we select the set of target
layers
$
    \mathcal{L}_{\rho} = \operatorname*{arg\,top\text{-}k}_{\ell \in \{1, \ldots, m\}} \;v_{\ell},
    k = \lceil \rho \, m \rceil,
$
i.e., the $\lceil \rho m \rceil$ layers with the highest scores. During unlearning, GradDiff updates are applied only to parameters $\theta_p$ with $p \in \mathcal{P}_i$ for $i \in \mathcal{L}_{\rho}$, while all remaining parameters are kept frozen.

\subsection{Application C: \textsc{WAG} for Mixed Precision Quantization}
\label{app:application_quantization}
We apply \textsc{WAG} to group-wise mixed-precision quantization, where the goal is to retain the most sensitive weight groups in high precision while quantizing the remaining ones. Following the group-quantization scheme of SliM-LLM \cite{huang2025slimllm}, we partition each weight matrix into contiguous groups of $128$ columns. Following the experimental setup of Sec.~\ref{sec:empirical}, we quantize a BFloat16 model trained on coding tasks under the SliM-LLM pipeline, keeping the top-$\kappa$ ($\kappa \in \{5\%, 10\%\}$) fraction of groups per matrix in FP16 and quantizing the remaining groups to 4-bit. The selection metric only decides \emph{which} groups are kept; all other quantization settings follow SliM-LLM.

Since \textsc{WAG} scores each parameter according to its local sensitivity under multiplicative perturbations, the collective sensitivity of a group with respect to the evaluation set can be estimated by averaging the absolute scores of its parameters. Formally, let $w_p$ denote the \textsc{WAG} value of parameter $p$, where gradients are computed on the evaluation set, and let $\mathcal{P}_g$ denote the set of parameters in group $g$. The $|\textsc{WAG}|$ criterion scores each group by
\begin{equation*}
    s_g^{\textsc{WAG}}
        = \frac{1}{|\mathcal{P}_g|} \sum_{p \in \mathcal{P}_g} |w_p|,
    \label{eq:quant-group-scores}
\end{equation*}
capturing groups whose parameters are highly sensitive in either direction. We compare $|\textsc{WAG}|$ against the two criteria: the group magnitude $s_g^{\mathrm{mag}} = \sum_{p \in \mathcal{P}_g} |W_p|$ and the SparseGPT OBS salience $s_g^{\mathrm{OBS}} = \sum_{i \in g} W_i^2 / [H^{-1}]_{ii}$, where $H$ is the layer Hessian estimated from calibration activations \cite{hassibi1993second, frantar2023sparsegpt}. To combine the $|\text{WAG}|$ and (OBS) signals, we min--max normalize each score over the groups of a matrix and sum,
\begin{equation*}
    s_g^{\mathrm{OBS}+|\textsc{WAG}|}
        = \widetilde{s}_g^{\mathrm{OBS}} + \widetilde{s}_g^{\textsc{WAG}},
    \qquad
    \widetilde{s}_g = \frac{s_g - \min_{g'} s_{g'}}{\max_{g'} s_{g'} - \min_{g'} s_{g'}}.
\end{equation*}
Given a preservation budget $\kappa \in \{5\%, 10\%\}$, each criterion keeps the top-$\kappa$ fraction of groups ranked by its score in FP16 and quantizes all remaining groups to 4-bit via the SliM-LLM GPTQ pipeline.

\subsection{Application D: \textsc{WAG} for Knowledge Editing}
\label{app:application_ke}
We next consider the locate-then-edit paradigm of knowledge editing, where we use \textsc{WAG} to identify the layers most relevant for editing. The intuition is that if we select the layers containing the largest number of parameters important to the new knowledge to be edited, the edit can be applied most effectively while minimally affecting unrelated knowledge. We evaluate this using the R-ROME \cite{gupta2024rebuilding} editing method on four datasets: \textit{ZSRE} \cite{levy2017zero, wang2024easyediteasytouseknowledgeediting}, \textit{WikiCounterFact} \cite{zhang2024comprehensive}, \textit{WikiBio} \cite{zhang2024comprehensive, hartvigsen2023aging}, and \textit{WikiRecent} \cite{zhang2024comprehensive, cohen2024evaluating}. We follow the evaluation framework adopted in the recent work \cite{datta2026golden}, where we select the layers based on a small proxy set of 10\% of the samples and evaluate the results on the test set consisting of the rest of the samples.

Since locate-then-edit methods \cite{wang2024easyediteasytouseknowledgeediting} target MLP layers, we restrict the selection to MLP layers and score them using \textsc{WAG}. In the knowledge editing application,\textsc{WAG} should reflect the sensitivity of the parameters to the specific knowledge being edited. Let an edit sample be a query-target pair $(x, y^{\ast})$, where $x$ is the prompt and $y^{\ast} = (y^{\ast}_1, \ldots, y^{\ast}_K)$ is the token sequence encoding the new knowledge. We define the gradient component of \textsc{WAG} via the autoregressive negative log-likelihood of the new knowledge given the query:
\begin{equation*}
    \mathcal{L}_{\mathrm{edit}}(\theta; x, y^{\ast})
    = -\sum_{k=1}^{K}
      \log p_{\theta}\!\left(y^{\ast}_{k} \mid x, \, y^{\ast}_{<k}\right),
    \label{eq:edit-loss}
\end{equation*}
so that the \textsc{WAG} value of parameter $p$, denoted $w_p$, is computed
with respect to $\nabla_{\theta_p} \mathcal{L}_{\mathrm{edit}}$, averaged
over the edit samples of the proxy set.

Since we aim to compare layers by the \emph{number} of important parameters, weighting each parameter uniformly, we rank parameters globally. Because $|\text{WAG}|$ captures sensitivity in both directions, each parameter is ranked according to the magnitude of its score, $r_p = \mathrm{rank}(|w_p|)$, where $\mathrm{rank}(\cdot)$ assigns rank $1$ to the lowest value. The score of an MLP layer $\ell$ is then the mean rank of its parameters:
$
    v_{\ell} = \frac{1}{|\mathcal{P}_{\ell}|}
    \sum_{p \in \mathcal{P}_{\ell}} r_p,
$
where $\mathcal{P}_{\ell}$ is the set of parameters in the MLP sublayer of
layer $\ell$, and the layer with the highest $v_{\ell}$, the highest concentration of important parameters is selected for editing. Layers are
identified once using the proxy set, and all evaluations are conducted on unseen test data.

\section{Code and Reproducibility}\label{app:code}
All the experiments were conducted on a Linux server with 6x NVIDIA DGX B200 GPUs with 192 GB VRAM/GPU. All the utilized LLMs are original (unquantized) open-source versions from HuggingFace. During generation, we used greedy decoding without sampling. To ensure reproducibility, all sources of randomness, including PyTorch, NumPy, and Python’s random module, were fixed using a seed of 42. For the experiments involving more than one run, we report the mean and SEM of a total of 3 runs with seeds 42, 43, and 44, respectively. For experiments of empirical analysis and Application C, we built on the MFT codebase \cite{zhang2026boostinglargelanguagemodels} and evaluations we conducted following their convention using OLMES \cite{gu2025olmesstandardlanguagemodel}. For Applications A and B, we followed the training and evaluation conventions of LayerIF \cite{askari2026layerif} and TOFU \cite{maini2024tofu} codebases, respectively. Our codebase for Application D was built upon the \textit{EasyEdit} \cite{wang2024easyediteasytouseknowledgeediting} framework and LGA codebase \cite{datta2026golden}, and evaluations of the edited models were performed following their standard conventions.

\end{document}